%% file: main_arxiv.tex
\documentclass[nohyperref]{article}

\usepackage{microtype}
\usepackage{graphicx}
\usepackage{subfigure}
\usepackage{booktabs} 
\usepackage{multirow}
\usepackage{caption}
\usepackage{color}
\usepackage{wrapfig}
\usepackage{float}
\usepackage{url}
\usepackage[nointegrals]{wasysym}
\usepackage[ruled,vlined]{algorithm2e}
\usepackage[normalem]{ulem}

\definecolor{amethyst}{rgb}{0.6, 0.4, 0.8}

\usepackage{hyperref}
\include{math_commands}

\newcommand{\Lcal}{\mathcal{L}}
\newcommand{\RR}{\mathbb{R}} 
\newcommand{\T}{^\top}
\newcommand{\bp}{\bar p}
\newcommand{\balpha}{\bar \alpha}
\newcommand{\halpha}{\hat \alpha}

\usepackage[accepted]{icml2023}

\usepackage{amsmath,amssymb,amsthm,amscd,amstext}
\usepackage{mathtools}
\usepackage[capitalize,noabbrev]{cleveref}

\theoremstyle{plain}
\newtheorem{theorem}{Theorem}[section]
\newtheorem{proposition}[theorem]{Proposition}
\newtheorem{lemma}[theorem]{Lemma}
\newtheorem{claim}[theorem]{Claim}

\theoremstyle{definition}

\newtheorem{assumption}[theorem]{Assumption}
\newtheorem{axiom}[theorem]{Axiom}
\theoremstyle{remark}

\newcommand{\ourmethod}{AuxiNash}
\newtheorem*{theorem*}{Theorem}

\makeatletter
\newcommand*\bigcdot{\mathpalette\bigcdot@{1.}}
\newcommand*\bigcdot@[2]{\mathbin{\vcenter{\hbox{\scalebox{#2}{$\m@th#1\bullet$}}}}}
\makeatother

\usepackage[textsize=tiny]{todonotes}

\begin{document}
\twocolumn[ 
\icmltitle{Auxiliary Learning as an Asymmetric Bargaining Game}
\icmlsetsymbol{equal}{*}

\begin{icmlauthorlist}
\icmlauthor{Aviv Shamsian}{equal,biu}
\icmlauthor{Aviv Navon}{equal,biu,aiol}
\icmlauthor{Neta Glazer}{biu,aiol}
\icmlauthor{Kenji Kawaguchi}{nus} \\
\icmlauthor{Gal Chechik}{biu,nvidia}
\icmlauthor{Ethan Fetaya}{biu}
\end{icmlauthorlist}

\icmlaffiliation{biu}{Bar-Ilan University, Ramat Gan, Israel}
\icmlaffiliation{aiol}{Aiola, Herzliya, Israel}
\icmlaffiliation{nvidia}{Nvidia, Tel-Aviv, Israel}
\icmlaffiliation{nus}{National University of Singapore}

\icmlcorrespondingauthor{Aviv Shamsian}{aviv.shamsian@live.biu.ac.il}
\icmlcorrespondingauthor{Aviv Navon}{aviv.navon@biu.ac.il}

\vskip 0.3in

]

\printAffiliationsAndNotice{\icmlEqualContribution}

\begin{abstract}
Auxiliary learning is an effective method for enhancing the generalization capabilities of trained models, particularly when dealing with small datasets. However, this approach may present several difficulties: (i) optimizing multiple objectives can be more challenging, and (ii) how to balance the auxiliary tasks to best assist the main task is unclear. In this work, we propose a novel approach, named \emph{\ourmethod{}}, for balancing tasks in auxiliary learning by formalizing the problem as a generalized bargaining game with asymmetric task bargaining power.
Furthermore, we describe an efficient procedure for learning the bargaining power of tasks based on their contribution to the performance of the main task and derive theoretical guarantees for its convergence. Finally, we evaluate \ourmethod{} on multiple multi-task benchmarks and find that it consistently outperforms competing methods.

\end{abstract}

\section{Introduction}
When training deep neural networks with limited labeled data, 
generalization can be improved by adding auxiliary tasks. In this approach, called \textit{Auxiliary learning} (AL), the auxiliary tasks are trained jointly with the main task, and their labels can provide a signal that is useful for the main task. AL is beneficial because auxiliary annotations are often easier to obtain than annotations for the main task. This is the case when the auxiliary tasks use self-supervision \citep{oliver2018realistic,HwangPKKHK20,achituve2021self}, or their annotation process is faster. For example, learning semantic segmentation of an image may require careful and costly annotation, but can be improved if learned jointly with a depth prediction task, whose annotations can be obtained at scale \citep{StandleyZCGMS20}. 


Auxiliary learning has a large potential to improve learning in the low data regime, but it gives rise to two main challenges:  Defining the joint optimization problem and performing the optimization efficiently. (1) First, given a main task at hand, it is not clear \textit{which} auxiliary tasks would benefit the main task and how tasks should be \textit{combined} into a joint optimization objective. For example, \citet{StandleyZCGMS20} showed that depth estimation is a useful auxiliary task for semantic segmentation but not the opposite. In fact, adding semantic segmentation as an auxiliary task harmed depth estimation performance. This suggests that even close-related tasks may interfere with each other. (2) Second, training with auxiliary tasks involves optimizing multiple objectives simultaneously; While training with multiple tasks can potentially improve performance via better generalization, it often underperforms compared to single-task models. 

Previous auxiliary learning research focused mainly on the first challenge: namely, weighting and combining auxiliary tasks \citep{lin2019adaptive}. The second challenge, optimizing the main task in the presence of auxiliary tasks, has been less explored. Luckily, this problem can be viewed as a case of optimization in \textit{Multi-Task Learning} (MTL). In MTL, there is extensive research on controlling optimization such that every task would benefit from the others. Specifically, several studies proposed algorithms to aggregate gradients from multiple tasks into a coherent update direction \citep{yu2020gradient,liu2021conflict,navon2022multi}. We see large potential in bringing MTL optimization ideas into auxiliary learning to address the optimization challenge.  

Here we propose a novel approach named \ourmethod{} that takes inspiration from recent advances in MTL optimization as a cooperative bargaining game \citep[Nash-MTL,][]{navon2022multi}. The idea is to view a gradient update as a shared resource, view each task as a player in a game, and have players compete over making the joint gradient similar to their own task gradient. In Nash-MTL, tasks play a symmetric role, since no task is particularly favorable. This leads to a bargaining solution that 
is proportionally fair across tasks. In contrast, task symmetry no longer holds in auxiliary learning, where there is a clear distinction between the primary task and the auxiliary ones.
As such, we propose to view auxiliary learning as an \emph{asymmetric bargaining game}. Specifically, we consider gradient aggregation as a cooperative bargaining game where each player represents a task with varying bargaining power. We formulate gradient update using asymmetric Nash bargaining solution which takes into account varying task preferences. By generalizing Nash-MTL to asymmetric games with \ourmethod, we can efficiently direct optimization solution towards various areas of the Pareto front. 

Determining the task preferences that result in optimal performance is a challenging problem for several reasons. First, the relationship between tasks can change during the optimization process, making it difficult to know in advance what preferences to use. This means that the process of preference tuning needs to be automated during training. Second, using a grid search to find the optimal preferences can be computationally expensive, and the complexity of the search increases exponentially with the number of tasks. To overcome these limitations, we propose a method for efficiently differentiating through the optimization process and using this information to automatically optimize task preferences during training. This can improve the performance of the primary task, and make it more efficient to find the optimal preferences.


\begin{figure}
    \centering
    \includegraphics[width=1.\linewidth]{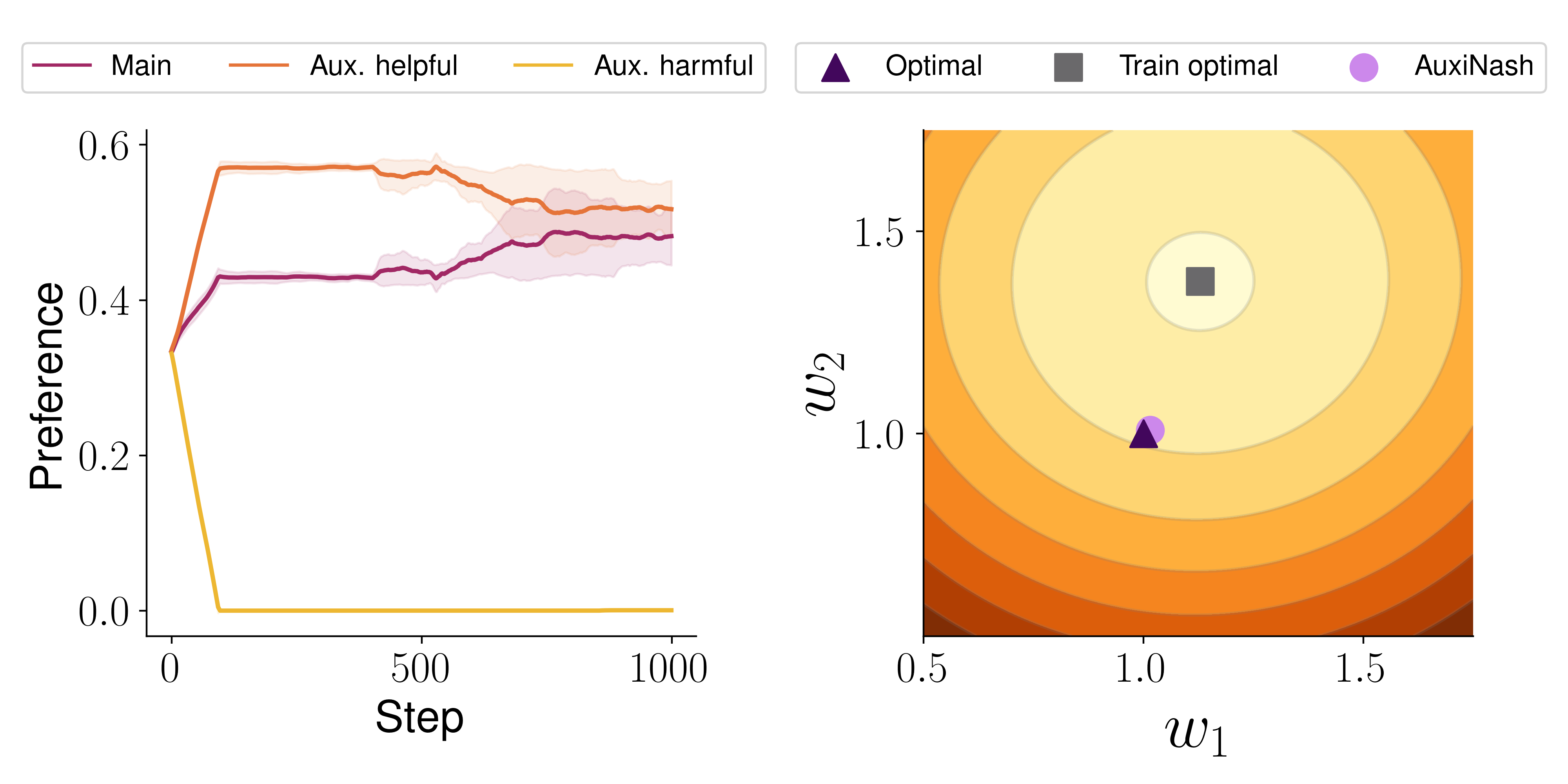}
    \caption{\textit{Illustrative example}: A regression problem in $\mathbb{R}^2$ with two auxiliary tasks, one helpful and one harmful. \ourmethod{} succeeds in using the helpful auxiliary task and disregards the harmful one, as demonstrated by how it learns to weigh different tasks: The left panel shows the preference vector $p$ during optimization. As a result, AuxiNash converges to a solution with large proximity to the optimal solution, far superior to that obtained from optimizing with the main tasks alone (right panel, darker colors indicate larger loss values). See Section~\ref{sec:illus} for further details. }
\label{fig:illustrative}
\end{figure}

We theoretically analyze the convergence of \ourmethod{} and show that even if the preference changes during the optimization process we are still guaranteed to converge to a Pareto stationary point. Finally, we show empirically on several benchmarks that \ourmethod{} achieves superior results to previous auxiliary learning and multi-task learning approaches. 

\paragraph{Contributions:} This paper makes the following contributions: (1) We introduce \ourmethod{} - a novel approach for auxiliary learning based on principles from asymmetric bargaining games. (2) We describe an efficient method to dynamically learn task preferences during training. (3) We theoretically show that \ourmethod{} is guaranteed to converge to a Pareto stationary point. (4) We conduct extensive experiments to demonstrate the superiority of \ourmethod{} against multiple baselines from MTL and auxiliary learning.

\section{Related Work}
 
\paragraph{Auxiliary Learning.} Learning with limited amount of training data is challenging since deep learning models tend to overfit to the training data and as a result can generalize poorly \citep{ying2019overview}. One  approach to overcome this limitation is using auxiliary learning \citep{chen2022auxiliary, kung2021efficient}. Auxiliary learning aims to improve the model performance on primary task by utilizing the information of related auxiliary tasks \citep{dery2022aang, chenmodule}. Most auxiliary learning approaches use a linear combination of the main and auxiliary losses to form a unified loss \citep{zhai2019s4l, wen2020entire}. Fine-tuning the weight of each task loss may be challenging as the search space of the grid search grows exponentially with the number of tasks. To find the beneficial auxiliary tasks, recent studies utilized the auxiliary task gradients and measure their similarity with the main task gradients \citep{lin2019adaptive, du2018adapting, shi2020auxiliary}. 
\citet{NavonAMCF21} proposed to learn a non-linear network that combines all losses into a single coherent objective function. 
\paragraph{Multi-task Learning.} In multi-task learning (MTL) we aim to solve multiple tasks by sharing information between them \citep{caruana1997multitask, ruder2017overview}, usually through joint hidden representation \citep{zhang2014facial, dai2016instance, pinto2017learning, Liu2019EndToEndML}. Previous studies showed that optimizing a model using MTL helps boost performances while being computationally efficient \citep{sener2018multi, chen2017gradnorm}. However, MTL presents a number of optimization challenges such as: conflicting gradients \citep{wang2020gradient, yu2020gradient}, and flatten areas in the loss landscape \citep{schaul2019ray}. These challenges may result with performance degradation compare with single task learning. Recent studies proposed novel architectures \citep{misra2016cross, hashimoto2017joint, Liu2019EndToEndML, Chen2020JustPA} to improve MTL while others focused on aggregating the gradients of the tasks such that it is agnostic to the optimized model \citep{liu2021towards, javaloy2021rotograd}. \citet{yu2020gradient} proposed to overcome the conflicting gradients problem by subtracting normal projection of conflicted task before forming an update direction. Most gradient based methods aim to minimize the average loss function. \citet{liu2021conflict} suggested an approach that will decrease every task loss in addition to the average loss function. The closest work to our approach is Nash-MTL \citep{navon2022multi}. The authors proposed a principled approach to dynamically weight the losses of different tasks by incorporating concepts from game theory. 
\paragraph{Bi-level Optimization.} Bi-Level Optimization (BLO) consists of two nested optimization problems~\citep{liao2018reviving,liu2021investigating, vicol2022implicit}. The outer optimization problem is commonly referred to as the upper-level problem, while the inner optimization problem is referred to as the lower-level problem \citep{sinha2017review}. BLO is widely used in a variety of deep learning applications, spanning hyper-parameter optimization \citep{foo2007efficient, mackay2019self}, meta learning \citep{franceschi2018bilevel}, reinforcement learning \citep{zhang2020bi, yang2019provably}, and multi-task learning \citep{liu2022auto, NavonAMCF21}. A common practice to derive the gradients of the upper-level task is using the implicit function theorem (IFT). However, applying IFT involves the calculation of an inverse-Hessian vector product which is infeasible for large deep learning models. Therefore, recent studies proposed diverse approaches to approximate the inverse-Hessian vector product. \citet{luketina2016scalable} proposed approximating the Hessian with the identity matrix, where other line of works used conjugate gradient (CG) to approximate the product \citep{foo2007efficient, pedregosa2016hyperparameter, Rajeswaran2019MetaLearningWI}. We use a truncated Neumann series and efficient vector-Jacobian products, as it was empirically shown to be more stable than CG \citep{liao2018reviving, lorraine2020optimizing, raghu2020teaching}.

\begin{figure}
    \centering
    \includegraphics[width=0.9\linewidth]{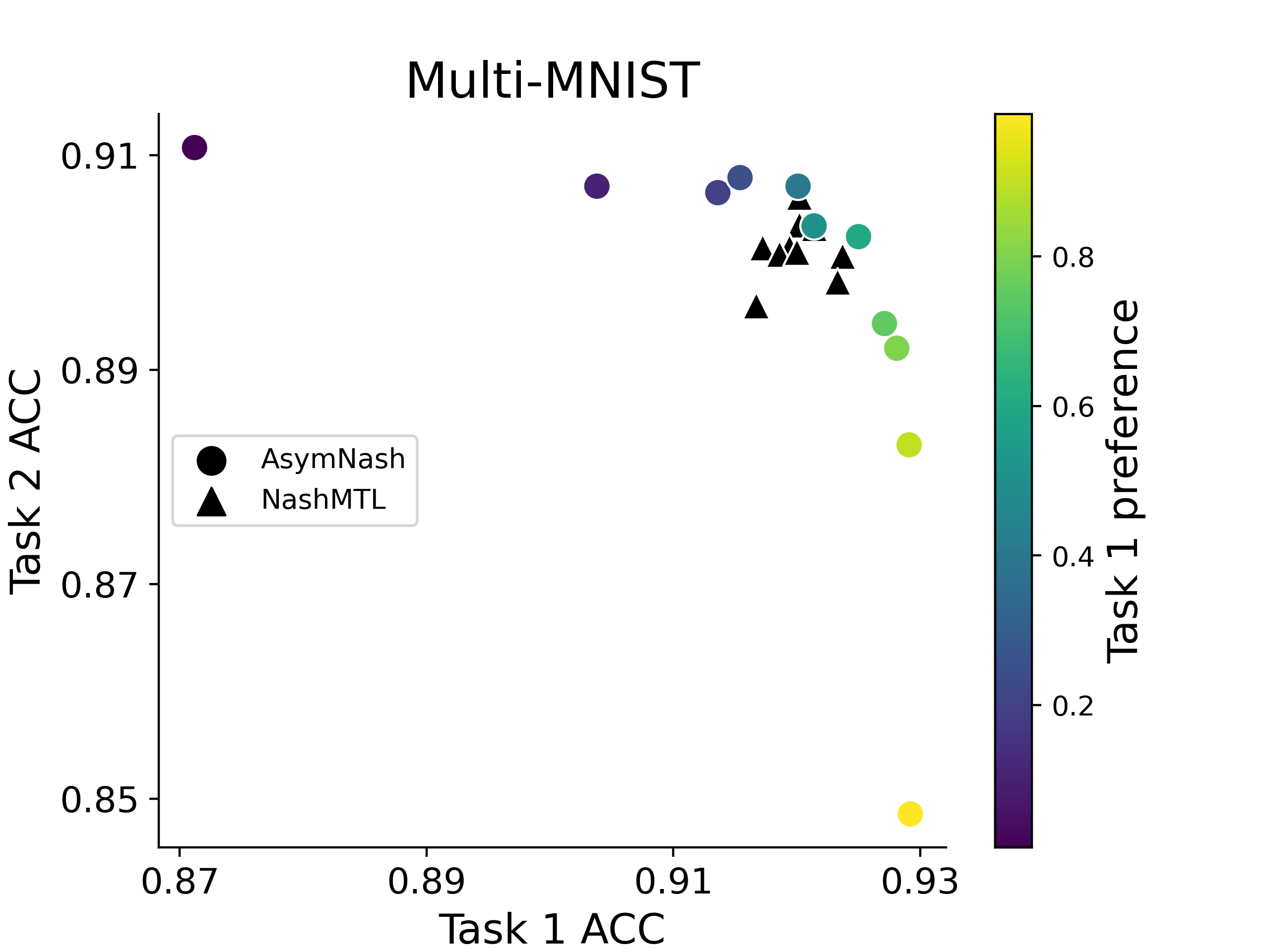}
    \caption{\textit{Task preferences}: By varying the preference vector $p$, we show that \ourmethod{} can control the trafe-off between tasks. Compared with Nash-MTL, AuxiNash achieves a wider range of diverse solutions, an important property for auxiliary learning. See Section~\ref{sec:multimnist} for further details.}
\label{fig:multimnist}
\end{figure}

\section{Background}

\subsection{Nash Bargaining Solution}\label{sec:nbs}

We will first give a quick introduction to cooperative bargaining games. In a bargaining game, a set of $K$ players jointly decide on an agreed-upon point in the set $A$ of all agreement points. If failing to reach an agreement, the game default to the disagreement point $D$. Each player $i\in [K]:=\{1, ...,K\}$ is equipped  with a  utility function $u_i:A\cup\{D\}\rightarrow\R$, which they wish to maximize. Intuitively, each player has a different objective, and each tries to only maximize their own personal utility. However,  we generally assume that there are points in the agreement set that are mutually beneficial to all players, compared to the disagreement point, and as such the players are incentivized to cooperate. The main question is on which point in the agreement set will they decide upon.\\

 Denote by $U=\{(u_1(x),...,u_K(x)):\,x\in A\}\subset\R^K$ the set of the utilities of all possible agreement points and $d=(u_1(D),...,u_K(D))$. The set $U$ is assumed to be convex and compact. Furthermore, we assume that there exists a point $u\in U$ for which $\forall i:u_i>d_i$. \citet{nash} showed that under these assumptions, there exists a unique solution to the bargaining game, which satisfies the following properties: Pareto optimality, symmetry, independence of irrelevant alternatives, and invariant to affine transformations \citep[see][and the supplementary material for more details]{game_theory}. This unique solution, referred to as the Nash Bargaining solution (NBS), is given by 
\begin{align}\label{eq:nash_barg}
    u^*=&\arg\max_{u\in U}\sum_i\log(u_i-d_i)\\
    &s.t.\,\,\forall i:\,u_i>d_i\nonumber.
\end{align}
As shown in~\citet{navon2022multi}, NBS properties are suitable for the multi-task learning setup, even if the invariance to affine transformations implies that gradient norms are ignored. The symmetry assumption, however, implies that each player is interchangeable which is not the case for  auxiliary learning. Naturally, our main concern is the main task and the auxiliaries are there to support it, not compete with it. Thus, we wish to discard the symmetry assumption for the auxiliary learning setup.


\subsection{Generalized Bargaining Game}

\citet{kalai1977nonsymmetric} generalized the NBS to the asymmetric case. First, define a preference vector to control the relative trade-off between tasks $p \in \R^K$ with $p_i>0$ and $\sum_i p_i=1$ (see Figure~\ref{fig:multimnist}). Similar to the symmetric case, the Generalized Nash Bargaining Solution (GNBS) maximizes a weighted product of utilities,
\begin{align}\label{eq:asym_nash_barg}
    u^*=&\arg\max_{u\in U}\sum_ip_i\log(u_i-d_i)\\
    &s.t.\,\,\forall i:\,u_i>d_i\nonumber \quad.
\end{align}
The symmetric case is a special case of GNBS with uniform preferences $p_i=1/K,~\forall i\in [K]$.

\subsection{Bargaining Game for Multi-task learning}

Recently,~\citet{navon2022multi} formalized multi-task learning as a bargaining game as follows. Let $\theta \in \RR^d$ denote the parameters of a network $f(\cdot; \theta)$. At each MTL optimization step, we search for an update direction $\Delta\theta$. Define the agreement set $U=\{\Delta\theta\ \mid \Vert \Delta\theta \Vert \leq 1 \}$ as the set of all possible update directions. The disagreement point $d$ is defined to be equal to zero, i.e., to stay with the current parameters and terminate the optimization process. Let $g_i$ denote the gradient of $\theta$ w.r.t. the loss of task $i$ (for each $i\in[K]$). The utility function for task $i$ is defined as $u_i(\Delta\theta)=g_i\T\Delta\theta$, i.e., the directional derivative in direction $\Delta\theta$. \citet{navon2022multi} assumed that the gradients $g_1,...,g_K$ are linearly independent if $\theta$ is not Pareto stationary, and we adopt that assumption in our analysis. Under this assumption, \citet{navon2022multi} show that the solution for the bargaining game at any non-Pareto stationary point $\theta$ is given by $\Delta\theta=\sum_i\alpha_ig_i$, where the weight vector $\alpha$ satisfies
\begin{align}
    G\T G\alpha=1/\alpha.
\end{align}
Here, $G$ is the $d\times K$ matrix whose $i$-th column is the $i$-th task gradient $g_i$, and $1/\alpha$ is the element-wise reciprocal.

\begin{figure*}[t]
    \centering
\includegraphics[width=1.\linewidth]{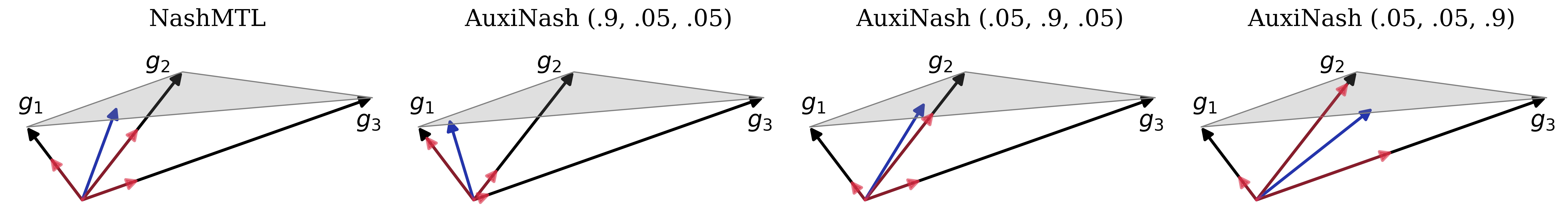}    
    \caption{\textit{Visualization of the update direction}: We show the update direction {(\textcolor{blue}{blue})} obtained by \ourmethod{} on three gradients in $\R^3$. We rescaled the update directions for better visibility, showing only the direction. We further show the size of the projection {(\textcolor{red}{red})} of the update to each gradient direction {(\textbf{black})}. By varying the preference vector, we observe the change in the obtained update direction. Importantly, we note that the effect on the update direction is non-trivial, as $p$ only affects the update implicitly through the bargaining solution $\alpha$.}
    \label{fig:update}
\end{figure*}

\section{Generalized Bargaining Game for Auxiliary Learning}

In this section, we first extend the result from~\citet{navon2022multi} to the asymmetric case. Next, we describe a method to learn the preference vector $p$. 

\subsection{Generalized Bargaining Solution} 
We prove the following claim, which generalizes the claim of ~\citet{navon2022multi} to asymmetric  games.


\begin{claim}
Let $p\in \mathbb{R}^K_+$ with $\sum_i p_i=1$. The solution to the generalized bargaining problem $\Delta\theta^*=\arg\max_{\Dt\in U}\sum_i p_i \log(\Dt\T g_i)$ is given by (up to scaling) $\sum_i\alpha_i g_i$ at any non-Pareto stationary point $\theta$, where $\alpha\in\mathbb{R}^K_+$ is the solution to $G\T G\alpha=p/\alpha$ where $p/\alpha$ is the element-wise reciprocal.  
\end{claim}
\begin{proof}
We define $F(\Delta\theta)=\sum_i p_i \log(\Dt\T g_i)$ and have $\nabla F=\sum_{i=1}^K \frac{p_i}{\Delta\theta^T g_i} g_i$. Note that for all $\Delta\theta$ with $u_i(\Delta\theta)>0$ for all $i\in [K]$ the utilities are monotonically increasing in $\Vert \Delta\theta\Vert$, hence the optimal solution lies on the boundary of $U$, and $\nabla F|_{\Delta\theta^*}$ is parallel to $\Delta\theta^*$. This implies that  $\sum_{i=1}^K \frac{p_i}{\Delta\theta\T g_i} g_i=\lambda \Delta\theta$ for some $\lambda > 0$. From the {linear} independence assumption, we have for the optimal solution $\Delta\theta=\sum_i\alpha_i g_i$, thus $\forall i, \Delta\theta\T g_i=\frac{p_i}{\lambda \alpha_i}$. Setting $\lambda=1$ (as we ignore scale), the solution to the bargaining game is reduced to finding $\alpha\in \mathbb{R}^K_+$ for which $\forall i,~\Delta\theta\T g_i=\sum_j\alpha_jg_j\T g_j=p_i/\alpha_i$. Equivalently, the optimal solution is given by $\alpha\in \mathbb{R}^K_+$ such that $G\T G\alpha= p/\alpha$ where $p /\alpha$ is the element-wise reciprocal.
\end{proof}

Given a preference vector $p$, we solve $G\T G\alpha= p/\alpha$ by expressing it as the solution to an optimization problem. We first solve a convex relaxation which we follow by a concave-convex procedure (CCP) 
\citep{yuille2003concave,lipp2016variations}, similar to \citet{navon2022multi} for solving $G\T G \alpha=p/\alpha$ w.r.t. $\alpha$. See Appendix~\ref{app:solving-the-game} for full details.


\subsection{Optimizing the Preference Vector}

The derivation in the previous section allows us to learn using a known preference vector $p$. Unfortunately, in most cases,  the preference vector is {not} known in advance. One simple solution is to treat the preferences $p_i$ as hyperparameters and set them via grid search. However, this approach has two significant limitations. First, as the number of hyperparameters increases, grid search becomes computationally expensive as it scales exponentially. Second, it is possible that the optimal preference vector would vary {during optimization}, hence using a fixed  $p$ would be sub-optimal.

To address these issues, we develop an approach for dynamically learning the task preference vector during training. This reduces the number of hyperparamters the user needs to tune to one (the preference update rate) and dynamically adjusts the preference to improve generalization. We do this by formulating the problem as bi-level optimization, which we discuss next.

Let $\gL_T$ denote the training loss and $\gL_V$ denote the \textit{generalization} loss, given by the loss of the main task on unseen data, i.e., $\gL_V=\ell_{main}^{val}$. In the auxiliary learning setup, we wish to optimize $p$ such that a network $f(\cdot;\theta)$ optimized with $\gL_T(\cdot; p)$ would minimize $\gL_V$. Formally, 
\begin{align*}
    p^*=\arg\min_{p} \gL_V(\theta^*(p)),~~\text{s.t.}~~\theta^*(p)=\arg\min_{\theta} \gL_T(\theta, p).
\end{align*}
Using the chain rule to get the derivative of the outer problem, we get
\begin{align*}
    \frac{\partial \Ls_V(p,\theta^*(p))}{\partial p} \hspace{-2pt}= \hspace{-2pt} \underbrace{\frac{\partial \Ls_V}{\partial p}}_{=0}+\frac{\partial \Ls_V}{\partial \theta} \frac{\partial \theta^*}{\partial p}
    \hspace{-2pt} = \hspace{-2pt}  \frac{\partial \Ls_V}{\partial \theta} \frac{\partial \theta^*}{\partial \alpha(p)} \frac{\partial \alpha(p)}{\partial p}.
\end{align*}
As we can compute $\frac{\partial \Ls_V}{\partial \theta}$ by simple backpropagation, the main challenge is to compute $\frac{\partial \theta^*}{\partial \alpha(p)}$ and $ \frac{\partial \alpha(p)}{\partial p}$.

To compute $\frac{\partial \theta^*}{\partial \alpha(p)}$ we can (indirectly) differentiate through the optimization process using the implicit function theorem (IFT)  \citep{liao2018reviving,lorraine2020optimizing,NavonAMCF21}:
\begin{align}
    \frac{\partial \theta^*}{\partial \alpha(p)}=-\left[\frac{\partial^2\gL_T}{\partial \theta\partial \theta\T}\right]^{-1} \frac{\partial^2\gL_T}{\partial\theta\partial\alpha(p)\T}.
\end{align}
Since computing the inverse Hessian directly is intractable, we use the algorithm proposed by \citet{lorraine2020optimizing} to efficiently estimate the inverse-Hessian vector product. This approach uses the Neumann approximation with an efficient vector-Jacobian product. Thus, we can efficiently approximate the first term, $\frac{\partial \Ls_V}{\partial \theta} \frac{\partial \theta^*}{\partial \alpha(p)}$. We note that in practice, as in customary, we do not optimize till convergence but perform a few gradient updates from the previous value. For further details, see ~\citet{vicol2022implicit} that recently examined how this affects the bi-level optimization process.
 
For the second term, $\frac{\partial \alpha(p)}{\partial p}$, we derive a simple analytical expression using the IFT in the following proposition:
\begin{proposition}\label{prop:dalpha-dp}
For any $(p,\alpha)$ satisfying $G\T G \alpha = p/\alpha$, there exists an open set $U \subseteq \RR^K$ containing $p$ such that there exists  a  continuously differentiable function $\halpha:U \to \RR^K$ satisfying all of the following properties: (1) $\halpha(p)=\alpha$, (2) $G\T G\halpha(\bar{p}) = \bar{p}/\halpha(\bar{p})$ for all $\bar{p} \in U$, and (3)
\begin{align}
\begin{split}
\frac{\partial \halpha(p)}{\partial p}=\left[G\T G+\Lambda_0\right]^{-1} \Lambda_1. 
\end{split}
\end{align}
Here $\Lambda_0 ,\Lambda_1\in \RR^{K\times K}$ are the diagonal matrices defined by $(\Lambda_0)_{ii}=p_i/\alpha_{i}^2 \in \RR$ and $(\Lambda_1)_{ii}=1/\alpha_{i} \in \RR$ for $i\in [K]$.
\end{proposition}
We refer the readers to Appendix~\ref{app:proofs} for the proof.

Putting everything together, we obtain the following efficient approximation,
\begin{align}\label{eq:hypergrad}
\small
\begin{split}
 \frac{\partial \Ls_V(p,\theta^*(p))}{\partial p}
&=-\frac{\partial \Ls_V}{\partial \theta} \left[\frac{\partial^2\gL_T}{\partial \theta\partial \theta\T}\right]^{-1} \frac{\partial^2\gL_T}{\partial\theta\partial\alpha(p)\T}
\times\\ 
& \qquad \left[G\T G+\Lambda_0\right]^{-1} \Lambda_1.
\end{split} 
\end{align}
We note that this approximation can be computed in a relatively efficient manner, with the cost of only several backpropagation operations to estimate the vector-Jacobian product (we use 3 in our experiments). We also note that the matrix $\left(G\T G+\Lambda_0\right)$ that we invert is of size $K\times K$, where $K$ is the number of tasks that is generally relatively small. 

 
In practice, we use a separate batch from the training set to estimate the generalization loss $\gL_V$. We further discuss this design choice and provide an empirical evaluation in Section~\ref{sec:aux-set}. During the optimization process, we alternate between optimizing $\theta$ and optimizing $p$. Specifically, we update $p$ once every $N_p$ optimization steps over $\theta$. We set $N_p=25$ in our experiments. The \ourmethod{} algorithm is summarized in Alg.~\ref{alg:auxinash}.


\begin{algorithm}[t]
    \caption{\ourmethod{}}\small\label{alg:auxinash}
    \begin{algorithmic}[H]
    \STATE {\bfseries Input:} $\theta$ -- initial parameter vector, $p$ -- initial preference vector, $\{\ell_i\}_{i=1}^K$ -- differentiable loss functions, $\eta, \eta_p$ -- learning rates
    \FOR{$T=1,...,N$}
    \FOR{$t=1,...,N_p$}
    \STATE Compute task gradients $g_i=\nabla_{\theta}\ell_i$
    \STATE Set $G$ the matrix with columns $g_i$
    \STATE Solve for $\alpha$: $G\T G\alpha=p/\alpha$
    \STATE Update the parameters $\theta\gets\theta-\eta G\alpha$
    \ENDFOR
    \STATE Evaluate $\nabla_p\gL_V$ using Eq.~\ref{eq:hypergrad}
    \STATE Update $p\gets p-\eta_p \nabla_p\gL_V$
    \ENDFOR
    \STATE {\bfseries Return:} $\theta$.
    \end{algorithmic}
\end{algorithm}

\begin{table*}[!t]
\setlength{\tabcolsep}{3pt}
\small
    \centering
    \caption{\textit{NYUv2}. Test performance for three tasks: semantic segmentation, depth estimation, and surface normal. Values are averages over 3 random seeds.}
    \vskip 0.11in
\resizebox{0.95\textwidth}{!}{%
\begin{tabular}{lccccccccccccccccccc}
\toprule\\
 &  &  & \multicolumn{2}{c}{Segmentation} &  & \multicolumn{2}{c}{Depth} &  & \multicolumn{8}{c}{Surface Normal} &  &  &  \\
 \cmidrule(lr){4-5} \cmidrule(lr){7-8} \cmidrule(lr){10-17}
 &  &  & \multirow{2}{*}{mIoU $\uparrow$} & \multirow{2}{*}{Pix Acc $\uparrow$} &  & \multirow{2}{*}{Abs Err $\downarrow$} & \multirow{2}{*}{Rel Err $\downarrow$} &  & \multicolumn{2}{c}{Angle Distance $\downarrow$} &  & \multicolumn{5}{c}{Within $t^\circ$  $\uparrow$}  & $\mathbf{\Delta\%} \downarrow$ \\
 \cmidrule(lr){10-11} \cmidrule(lr){13-17} \cmidrule(lr){19-19} \cmidrule(lr){20-20}
 &  &  &  &  &  &  &  &  & Mean & Median &  & 11.25 &  & 22.5 &  & 30 \\
 \midrule
 & \multicolumn{2}{c}{STL} & $38.30$ & $63.76$ &  & $0.6754$ & $0.2780$ &  & $25.01$ & $19.21$ &  & $30.14$ &  & $57.20$ &  & $69.15$ \\
  \midrule
 & \multicolumn{2}{c}{LS} & $38.43$ & $64.36$ &  & $0.5472$ & $0.2184$ &  & $29.57$ & $25.42$ &  & $20.50$ &  & $44.85$ &  & $58.20$ & $~~~8.69$ \\
 
 & \multicolumn{2}{c}{PCGrad} & $39.25$ & $64.95$ &  & $0.5389$ & $0.2141$ &  & $28.66$ & $24.26$ &  & $21.99$ &  & $47.00$ &  & $60.31$ & $~~~5.66$ \\
 
 & \multicolumn{2}{c}{CAGrad} & $39.25$ & $65.15$ &  & $0.5385$ & $0.2155$ &  & $26.11$ & $20.95$ &  & $26.96$ &  & $53.66$ &  & $66.37$ & $-1.46$ \\
 & \multicolumn{2}{c}{Nash-MTL} & $39.83$ & $66.00$ &  & $0.5235$ & $0.2075$ &  & $25.32$ & $19.87$ &  & $28.86$ &  & $55.87$ &  & $68.27$ & $-4.76$ \\
 \midrule
 
 & \multicolumn{2}{c}{GCS} & $38.96$ & $64.35$ && $0.5769$ & $0.2293$ && $29.57$ & $25.53$ && $20.64$ && $44.68$ && $57.99$ & $~~~9.54$ \\
 & \multicolumn{2}{c}{OL-AUX} & $40.51$ & $65.49$ && $0.6652$ & $0.2614$ && $\mathbf{24.65}$ & $\mathbf{18.72}$ && $\mathbf{30.92}$ && $\mathbf{58.37}$ && $\mathbf{70.12}$ & $-2.88$\\
 & \multicolumn{2}{c}{AuxiLearn} & $38.63$ & $64.20$ && $0.5415$ & $0.2173$ && $29.98$ & $25.29$ && $20.03$ && $43.94$ && $57.17$ & $~~~9.15$ &  \\
 \midrule
 & \multicolumn{2}{c}{\ourmethod{} (ours)} &
 
 $\mathbf{40.79}$ & $\mathbf{66.79}$ &  & $\mathbf{0.5092}$ & $\mathbf{0.2042}$ & & $24.90$ & $19.31$ &  & $29.83$ &  & $57.07$ &  & $69.27$ & $\mathbf{-6.80}$\\
 \bottomrule
\end{tabular}%
}
\label{tab:nyu}
\end{table*}

\section{Analysis}
We analyze the convergence properties of our proposed method in nonconvex optimization. We adopt the following three assumptions from~\citet{navon2022multi}:
\begin{assumption}\label{assump:independence}
We assume that for a sequence $\{\theta^{(t)}\}_{t=1}^\infty$ generated by our algorithm, the set of the gradient vectors $g_1^{(t)},...,g_K^{(t)}$ at any point on the sequence and at any partial limit are linearly independent unless that point is a Pareto stationary point. 
\end{assumption}
\begin{assumption}\label{assump:diff}
We assume that all loss functions are differentiable, bounded below and that all sub-level sets are bounded. The input domain is open and convex.
\end{assumption}
\begin{assumption}\label{assump:L-smooth}
We assume that all the loss functions are $L$-smooth,
\begin{equation}
    \Vert \nabla\ell_i(x)-\nabla\ell_i(y)\Vert \leq L\Vert x-y\Vert.
\end{equation}
\end{assumption}
Since even single-task non-convex optimization might only admits convergence to a stationary point, the following theorem proves convergence to a Pareto stationary point when both $\theta$ and $p$ are optimized concurrently:
\begin{theorem}\label{th:novergence}
Suppose that Assumptions  \ref{assump:independence}, \ref{assump:diff}, and  \ref{assump:L-smooth} hold.  Let $\{\theta^{(t)}\}_{t=1}^\infty$ be the sequence generated by the update rule $\theta^{(t+1)}=\theta^{(t)}-\mu^{(t)}\Dt^{(t)}$ where $\Dt^{(t)}=\sum_{i=1}^K\alpha^{(t)}_ig_i^{(t)}$ is the weighted Nash bargaining solution $(G^{(t)})\T G^{(t)}\alpha^{(t)}=p^{(t)}/\alpha^{(t)}$ where $p^{(t)}$ can be any discrete distribution. Set $\mu^{(t)}=\frac{1}{K}\sum_{i=1}^{K}p^{(t)}_{i}(L\alpha^{(t)}_i)^{-1}$. The sequence $\{\theta^{(t)}\}_{t=1}^\infty$ has a subsequence that converges to a Pareto stationary point $\theta^*$. Moreover all the loss functions $(\ell_1(\theta^{(t)}),...,\ell_K(\theta^{(t)}))$ converge to $(\ell_1(\theta^*),...,\ell_K(\theta^*))$.
\end{theorem}
See full proof in Appendix~\ref{app:proofs}.


\section{Experiments}

In this section, we compare \ourmethod{} with different approaches from multi-task and auxiliary learning. We use variety of datasets and learning setups to demonstrate the superiority of \ourmethod{}. To encourage future research and reproducibility, we make our source code publicly available \textcolor{magenta}{\url{https://github.com/AvivSham/auxinash}}. Additional experimental results and details are provided in Appendix~\ref{app:exp_details}.

\paragraph{Baselines.} We compare \ourmethod{} with natural baselines from recent auxiliary and multi-task learning works. The compared methods includes (1) Single-task learning (STL), which trains a model using the main task only; (2) Linear scalarization (LS) that minimizes the sum of losses $\sum_k\ell_k$; (3) GCS \citep{du2018adapting}, an auxiliary learning approach that uses gradient similarity between primary and auxiliary tasks; (4) OL-AUX \citep{lin2019adaptive}, an auxiliary learning approach that adaptively changes the loss weight based on the gradient inner product w.r.t the main task; (5) AuxiLearn \citep{NavonAMCF21}, an auxiliary learning approach that dynamically learns non-linear combinations of different tasks; (6) PCGrad \citep{yu2020gradient}, an MTL method that removes gradient components that conflict with other tasks; (7) CAGrad \citep{liu2021conflict}, an MTL method that optimizes for the average loss while explicitly controlling the minimum decrease rate across tasks; (8) Nash-MTL \citep{navon2022multi},  an MTL approach that is equivalent to \ourmethod{} but with a fixed $p_i=1/K$ weighting.

\paragraph{Evaluation} We report the common evaluation metrics for each task. Since MTL methods treat each task equally, and these may vary in scale, we also report the overall relative multi-task performance $\Delta\%$. $\Delta\%$ is defined as the performance drop compared to the STL performance. Formally, $\Delta\%=\frac{1}{K}\sum_{k=1}^{K}(-1)^{\delta_k}(M_{m,k} - M_{b,k} )/ M_{b,k}$. We denote $M_{b,k}$ and $M_{m,k}$ as the performance of STL and the compared method on task $k$, respectively. $\delta_k=0$ if a lower value is better for the metric $M_k$ and $1$ otherwise \citep{maninis2019attentive}. In all experiments, we report the mean value based on $3$ random seeds.    

It is important to note that for MTL models, we present the results of a single model trained on all tasks. For auxiliary learning methods, we trained a unique model per task, treating it as the main task and using the remaining tasks as auxiliaries.

\subsection{Illustrative Example}\label{sec:illus}

We start with an illustrative example, showing that \ourmethod{} can utilize helpful auxiliaries while ignoring harmful ones. 

We adopt a similar problem setup as in~\citet{NavonAMCF21} and consider a regression problem with parameters $W^T=(w_1, w_2)\in\RR^2$, fully shared among tasks. The optimal parameters for the main and helpful auxiliary tasks are $W^\star$, while the optimal parameters for the harmful auxiliary are $\tilde{W}\neq W^\star$. The main task is sampled from a Normal distribution $N({W^\star}^T x, \sigma_{\text{main}})$, with $\sigma_{\text{main}} > \sigma_{\text{h}}$ where $\sigma_{\text{h}}$ denotes the standard deviation for the noise of the helpful auxiliary. 

The change in the task preference throughout the optimization process is depicted in the left panel of Figure~\ref{fig:illustrative}. \ourmethod{} identify the helpful tasks and fully ignore the harmful ones. In addition, 
Figure~\ref{fig:illustrative} right panel presents the main task's loss landscape, along with the optimal solution ($W^\star$, marked 	$\blacktriangle$), the optimal training set solution of the main  task alone (${\scriptstyle \blacksquare}$) and the solution obtained by \ourmethod{} (marked ${\small\CIRCLE}$). While using the training data alone with no auxiliary information yields a solution that generalizes poorly, \ourmethod{} converges to a solution with large proximity to the optimal solution $W^\star$, 

\begin{table}[!t]
\setlength{\tabcolsep}{2pt}
    \centering
    \caption{\textit{Cityscapes}. Test performance for three tasks: 19-class semantic segmentation, 10-class part segmentation, and disparity. }
    \vskip 0.11in
\scriptsize
\begin{tabular}{lcccccccccccc}
\toprule\\
 &  &  & \multicolumn{2}{c}{Semantic Seg.} &  & \multicolumn{2}{c}{Part Seg.} &  & {Disparity}  \\
 \cmidrule(lr){4-5} \cmidrule(lr){7-8} \cmidrule(lr){10-10}
 &  &  & mIoU $\uparrow$ & Pix Acc $\uparrow$ & & mIoU $\uparrow$ & Pix Acc $\uparrow$  &  &  Abs Err $\downarrow$  & $\mathbf{\Delta \%} \downarrow$ \\
 \midrule
 & \multicolumn{2}{c}{STL} & $48.64$ & $91.01$ &  & $53.60$ & $97.62$ &  & $1.108$ &  \\
  \midrule
 & \multicolumn{2}{c}{LS} & $37.66$ & $88.63$ &  & $40.92$ & $96.98$ &  & $1.105$ & $~~~9.84$ \\
 & \multicolumn{2}{c}{PCGrad} & $39.10$ & $89.31$ &  & $41.71$ & $97.14$ &  & $1.133$ & $~~~9.28$ \\
& \multicolumn{2}{c}{CAGrad} & $39.45$ & $89.04$ &  & $51.95$ & $97.54$ &  & $1.098$ & $~~~4.66$ \\
 & \multicolumn{2}{c}{Nash-MTL} & $51.14$ & $91.59$ &  & $56.99$ & $97.87$ &  & $1.066$ & $-3.23$ \\
 \midrule
 
 & \multicolumn{2}{c}{GCS} & $37.45$ & $88.62$ && $41.14$ & $96.97$ &&  $1.124$ & $~10.19$\\
 & \multicolumn{2}{c}{OL-AUX} & $27.63$ & $89.34$ && $51.12$ & $97.52$ && $1.397$ & $~15.16$\\
 & \multicolumn{2}{c}{AuxiLearn} & $36.18$ & $88.24$ && $40.51$ & $96.95$ && $1.141$ & $~11.3$ \\
 \midrule
 & \multicolumn{2}{c}{\ourmethod{} } & $\mathbf{52.52}$ & $\mathbf{91.91}$ &  & $\mathbf{58.53}$ & $\mathbf{97.93}$ &  & $\mathbf{1.027}$ & $\mathbf{-5.15}$ \\
 \bottomrule
\end{tabular}%
\label{tab:cityscapes}
\end{table}

\subsection{Controlling Task Preference} \label{sec:multimnist}
In this section, we wish to better understand the relationship between the preference $p$ and the obtained solution.  We note the preference vector only implicitly affects the optimization solution through the bargaining solution $\alpha$.

Here, we show that controlling the preference vector can be used to steer the optimization outcome to different parts of the Pareto front, compared to the NashMTL baseline. 
We consider MTL setup with 2 tasks and use the Multi-MNIST \citep{sabour2017dynamic} dataset. In Multi-MNIST two images from the original MNIST dataset are merged into one by placing one at the top-left corner and the other at the bottom-right corner. The tasks are defined as image classification of the merged images. 
We run \ourmethod{} $11$ times with varying preference vector values $p$ and fix it throughout the training. For both tasks we report the classification accuracy. For Nash-MTL we run the experiments with different seed values. For both methods we train a variant of LeNet model for $50$ epochs with Adam optimizer and $1e-4$ as the learning rate. 

Figure~\ref{fig:multimnist} shows the results. \ourmethod{} reaches a diverse set of solutions across the Pareto front while Nash-MTL solutions are all relatively similar due to its symmetry property. 

\subsection{Analyzing the Effect of Auxiliary Set}\label{sec:aux-set}


\begin{table}[h]
\centering
    \caption{\textit{The effect of auxiliary set:} We report the mean IoU, along with the $\%$ change w.r.t STL performance.}
    \vskip 0.11in
\scriptsize
\begin{tabular}{lcccc}
\toprule
 &  \multicolumn{2}{c}{Cityscapes} & \multicolumn{2}{c}{NYUv2} \\
\cmidrule(lr){2-3} \cmidrule(lr){4-5}
 &  mIoU $\uparrow$  & Change $\%$ $\uparrow$ &  mIoU $\uparrow$  & Change $\%$ $\uparrow$ \\
\midrule
STL & $48.64$ &  & $38.30$ \\ 
STL Partial & $45.97$  & $-5.48$ &  $36.54$ &  $-4.59$\\
\midrule
\ourmethod{} &  $52.52$ & $~~~7.97$ & $40.79$ & $~~~6.50$ \\
\ourmethod{} Aux. Set & $51.81$  & $~~~6.51$ & $38.94$ & $~~~1.67$ \\
\bottomrule
\end{tabular}
\label{tab:aux_set}
\end{table}

One important question is on what data to evaluate the generalization loss $\mathcal{L}_V$. It would seem intuitive that one would need a separate validation set for that since estimating $\mathcal{L}_V$ on the training data may be biased. In practice, some previous works use a held-out auxiliary set \citep{NavonAMCF21}, while others use a separate batch from the training set \citep{liu2019self,liu2022auto}. While using an auxiliary set might be more intuitive, it requires reducing the available amount of training data which can be detrimental in the low-data regime we are interested in. 

We empirically evaluate this using the NYUv2~\citep{silberman2012indoor} and Cityscapes~\citep{Cordts2016Cityscapes} datasets. See Section~\ref{sec:scene-under} for more details. We choose semantic segmentation as the main task for both datasets. We compare the following methods (i) \textit{STL}: single task learning using the main task only, (ii) \textit{STL Partial}: STL using only $90\%$ of the training data, (iii) \textit{\ourmethod{}}: our proposed method where we optimize the preference vector using the entire training set, (iv) \textit{\ourmethod{} Aux. Set}: our proposed method, where we optimize the preference vector using $10\%$ of the data, allocated from the training set.

We report the mean-IoU metric (higher is better) along with the relative change from the performance of the STL method. The results suggest that the drawback of sacrificing some of the training data overweighs the benefit of using an auxiliary set. This result aligns with the observation in~\citet{liu2022auto}.

\subsection{Scene Understanding}\label{sec:scene-under}

We follow the setup from \citet{Liu2019EndToEndML, liu2022auto, navon2022multi} and evaluate \ourmethod{} on the NYUv2 and Cityscapes datasets \citep{silberman2012indoor, Cordts2016Cityscapes}. 
The indoor scene NYUv2 dataset \citep{silberman2012indoor} contains 3 tasks: 13 classes semantic segmentation, depth estimation, and surface normal prediction. The dataset consists of $1449$ RGBD images captured from diverse indoor scenes. 

We also use the Cityscapes dataset \cite{Cordts2016Cityscapes} with 3 tasks (similar to~\citet{liu2022auto}): 19-class semantic segmentation, disparity (inverse depth) estimation, and 10-class part segmentation~\citep{de2021part}. To speed up the training phase, all images and label maps were resized to $128 \times 256$. 

For all methods, we train SegNet \citep{badrinarayanan2017segnet}, a fully convolutional model based on VGG16 architecture. We follow the training and evaluation procedure from \citet{Liu2019EndToEndML, liu2022auto, navon2022multi} and train the network for 200 epochs with Adam optimizer \citep{Kingma2015AdamAM}. We set the learning rate to $1e-4$ and halved it after $100$ epochs. 

The results are presented in Table~\ref{tab:nyu} and Table~\ref{tab:cityscapes}. Observing the results, we can see our approach \ourmethod{} outperforms other approaches by a significant margin. It is also important to note that several methods achieve a positive $\%\Delta$ score, meaning they performed worse than simply ignoring the auxiliary tasks and training on the main task alone. We believe this is due to the difficulties presented by optimizing with multiple objectives. 


\subsection{Semi-supervised Learning with SSL Auxiliaries}

\begin{table}[h]
    \setlength{\tabcolsep}{3pt}
    \centering
    \caption{\textit{CIFAR10-SSL}. Test performance for classification with a varying number of labeled data. Values are averages over 3 random seeds.}
    \vskip 0.11in
    \resizebox{0.95\linewidth}{!}{%
    \begin{tabular}{lcc}
    \toprule
    & CIFAR10-SSL-5K & CIFAR10-SSL-10K \\
    \midrule
    STL & $79.31 \pm 0.31$ & $83.75 \pm 0.18$\\
    \midrule
    LS & $83.17 \pm 0.54$ & $86.16\pm 0.39$ \\
    PCGrad & $82.71 \pm 0.16$ &  $86.17\pm 0.34$ \\
    CAGrad & $85.89\pm 0.63$  & $87.82 \pm 0.28$    \\
    Nash-MTL & $\mathbf{86.69\pm 0.14}$ & $\mathbf{88.68 \pm 0.14}$  \\
    \midrule
    GCS & $83.09\pm 0.34$ & $86.47\pm 0.56$  \\
    OL-AUX & $81.44 \pm 1.06$ & $85.49 \pm 0.73$ \\
    AuxiLearn & $82.83 \pm 0.57$ & ${85.52 \pm 0.57}$ \\
    \midrule
    \ourmethod{} (ours) & $\mathbf{87.01 \pm 0.52}$ & $\mathbf{88.81 \pm 0.34}$ \\
    \bottomrule
    \end{tabular}
    }
    \label{tab:cifar_ssl}
\end{table}

In semi-supervised learning one generally trains a model with a small amount of labeled data, while utilizing self-supervised tasks as auxiliaries to be optimized using unlabeled training data. 

We follow the setup from \citet{shi2020auxiliary} and evaluate \ourmethod{} on a Self-supervised Semi-supervised Learning setting \citep{zhai2019s4l}. We use CIFAR-10 dataset to form 3 tasks. We set the supervised classification as the main task along with two self-supervised learning (SSL) tasks used as auxiliaries: 
(i) Rotation (ii) Exempler-MT. In Rotation, we randomly rotate each image by $[0^{\circ}, 90^{\circ}, 180^{\circ}, 270^{\circ}]$ and optimize the network to predict the angle. In Exempler-MT we apply a combination of three transformations: horizontal flip, gaussian noise, and cutout. Similarly to contrastive learning, the model is trained to extract invariant features by encouraging the original and augmented images to be close in their feature space. For the supervised task we randomly allocate samples from the training set. We repeat this experiment twice with $5K$ and $10K$ labeled training examples.
The results are presented in Table~\ref{tab:cifar_ssl}. \ourmethod{} significantly outperforms most baselines.

\subsection{Audio Classification}

\begin{table}[h]
\centering
\centering
    \caption{\textit{Speech Commands}. Test accuracy for speech classification, for models trained with $1000$ and $500$ training examples.}
    \vskip 0.11in
\begin{tabular}{lcc}
\toprule
& SC-500 &  SC-1000 \\
\midrule
STL   & $95.8 \pm 0.1 $ & $96.4 \pm 0.1$ \\
\midrule
LS  & $95.7 \pm 0.2$ & $96.7 \pm 0.1 $ \\
PCGrad  &  $95.7 \pm 0.1$ & $96.7 \pm 0.1 $ \\
CAGrad  & $95.7 \pm 0.2 $   & $95.7 \pm 0.1 $ \\
Nash-MTL  & $95.7 \pm 0.2 $  & $96.6 \pm 0.3 $\\
\midrule
GCS  & $96.3 \pm 0.1$  & $96.9 \pm 0.1 $\\
OL-AUX  & $96.2 \pm 0.2$ & $96.9 \pm 0.1 $\\
AuxiLearn  & $96.0 \pm 0.1 $ & $97.0 \pm 0.1 $\\
\midrule
\ourmethod{} (ours) & $\mathbf{96.4 \pm 0.1}$ & $\mathbf{97.2 \pm 0.1}$\\
\bottomrule
\end{tabular}
\label{tab:sc}
\end{table}

We evaluate AuxiNash on the speech commands (SC) dataset \citep{warden2018speech}, which consists of $\sim50$K speech samples of specific keywords. The data contains 30 different keywords, and each speech sample is one second long. We use a subset of the SC containing audio samples for only the 10 numbering keywords (zero to nine). As a pre-processing step, we use a short-time Fourier transform (STFT) to extract a spectrogram for each example, which we then fed to a convolutional neural network (CNN). We evaluate \ourmethod{} on 10 one-vs-all binary classification tasks. 
We repeat the experiment with a training set of sizes 500 and 1000. The results are presented in Table~\ref{tab:sc}.

\subsection{Learned Preferences}

\begin{figure*}[h]
    \centering
    \includegraphics[width=0.9\linewidth]{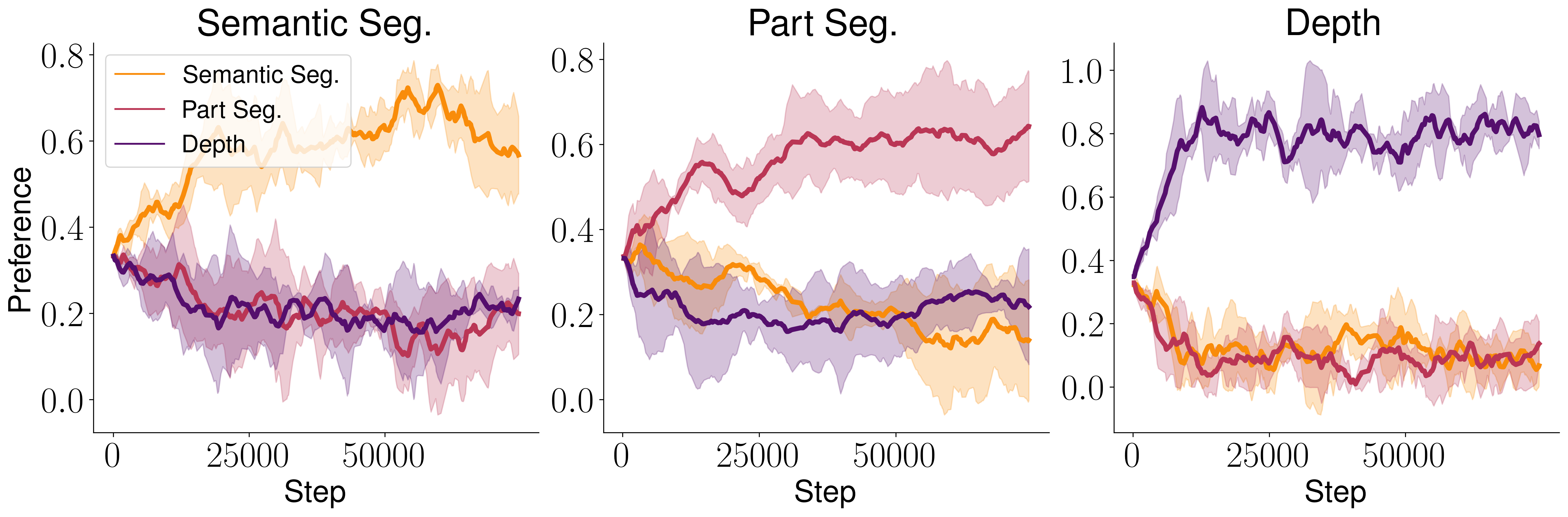}
    \caption{Learned task preference for Cityscapes dataset. The title of each panel indicates the main task.}
    \label{fig:pref}
\end{figure*}

To better understand the learned preference and its dynamics, we observe the change in preference vectors throughout the optimization process. We use the scene understanding experiment of Section~\ref{sec:scene-under}, specifically the Cityscapes dataset. The learned preferences are presented in Figure~\ref{fig:pref}. The title of each panel indicates the main task. Here, \ourmethod{} learns preferences that are aligned with our intuition, i.e., the main task's preference is the largest. 

\section{Limitations}
Auxiliary learning methods, while valuable in various domains, are not without their limitations. One potential limitation is designing effective auxiliary tasks that truly capture the underlying structure and knowledge of the main task can be challenging. Another limitation is the potential for negative transfer. If the auxiliary task is not sufficiently aligned with the main task, the learned representations may not be beneficial for improving performance. Additionally, auxiliary learning methods may suffer from scalability issues when dealing with large-scale datasets or complex problems. Finally, auxiliary learning methods might be sensitive to the choice of hyperparameters, architecture, or training procedures, making them less robust and generalizable across different tasks and datasets. Overcoming these limitations remains an ongoing research endeavor to unlock the full potential of auxiliary learning methods.

\section{Conclusion and Future Work}\label{sec:con}
In this work, we formulate  auxiliary learning as an asymmetric bargaining game and use game-theoretical tools to derive an efficient algorithm. We adapt and generalize recent advancements in multi-task learning to auxiliary learning and show how they can be automatically tuned to get a significant improvement in performance. 

We evaluated \ourmethod{} on multiple datasets with different learning setups and show that it outperforms previous approaches by a significant margin. Across all experiments, it is noticeable that MTL methods perform better than auxiliary learning ones although the former treat equally the primary task and the auxiliary tasks. We suspect that this is caused by conflicting gradients and by the fact that gradient norms may vary significantly across tasks. 
These results emphasize the connection between auxiliary learning and multi-task optimization. In many examples, the benefit of the auxiliary task was diminished or even completely negated by poor optimization. Thus, we suggest that auxiliary learning research should be closely aligned with MTL optimization research to effectively utilize auxiliary tasks.   

\section{Acknowledgements}

This study was funded by a grant to GC from the Israel Science Foundation (ISF 737/2018), and by
an equipment grant to GC and Bar-Ilan University from the Israel Science Foundation (ISF 2332/18).
AN and AS are supported by a grant from the Israeli higher-council of education, through the Bar-Ilan data science institute (BIU DSI). 

\bibliography{ref}
\bibliographystyle{icml2023}

\newpage
\appendix
\onecolumn
\section{Nash Bargaining Solution Details}
 The set $U$ is assumed to be convex and compact. Furthermore, we assume that there exists a point $u\in U$ for which $\forall i:u_i>d_i$. 
 \citet{nash} showed that under these assumptions, the two-player bargaining problem has a unique solution
that satisfies the following properties or axioms
: Pareto optimality, symmetry, independence of irrelevant alternatives, and invariant to affine transformations. This was later extended to multiple players \citep{game_theory}.
\begin{axiom}
\textbf{Pareto optimality:} The agreed solution must not be dominated by another option, i.e. there cannot be any other agreement that is better for at least one player and not worse for any of the players.
\end{axiom}
\begin{axiom}
\textbf{Symmetry:} The solution should be invariant to permuting the order of the players.
\end{axiom}

\begin{axiom}
\textbf{Independence of irrelevant alternatives (IIA):} If we enlarge the of possible payoffs to 
$\tilde{U}\supsetneq U$, and the solution is in the original set $U$, $u^*\in U$, then the agreed point when the set of possible payoffs is $U$ will stay $u^*$.
\end{axiom}
\begin{axiom}
\textbf{Invariance to affine transformation:} If we transform each utility function $u_i(x)$ to $\tilde{u}_i(x)=c_i\cdot u_i(x)+b_i$ with $c_i>0$ then if the original agreement had utilities $(y_1,...,y_k)$ the agreement after the transformation has utilities $(c_1y_1+b_1,...,c_ky_k+b_k)$
\end{axiom}
\section{Proofs}\label{app:proofs}

\begin{proof}[Proof of Proposition \ref{prop:dalpha-dp}]
Define a function $F(\bp,\balpha)=G\T G \balpha - \bp/\balpha \in \RR^K$ where and  $\bp \in \RR_{+}^K$ and $\balpha\in \RR_{+}^K$ are independent variables of $F$ with $\RR_{+}$ being the set of strictly positive real numbers. Here,  $\bp/\balpha$ represents the coordinate-wise operation, i.e., $\bp/\balpha \in \RR^K$ with $(\bp/\balpha)_i=\bp_{i}/\balpha_i \in \RR$ for $i\in[K]$. Then, we have
\begin{align*}
\frac{\partial F(\bp,\balpha)}{\partial \balpha}\Big\vert_{(\bp,\balpha)=(p,\alpha)}=G\T G - \frac{\partial (\bp/\balpha)}{\partial \balpha} \Big\vert_{(\bp,\balpha)=(p,\alpha)}=G\T G+\Lambda_0,
\end{align*}
where the last equality follows from $\frac{\partial \bp/\balpha}{\partial \balpha}=-\Lambda_0$ since   
$$
\frac{\partial( \bp/\balpha)_{i}}{\partial \balpha_{j}}= \begin{cases}\frac{\partial( \bp_{j}/\balpha_{j})_{}}{\partial \balpha_{j}} & \text{if } i=j\\
0 & \text{otherwise} \\
\end{cases} = \begin{cases}-\bp_i/\balpha_{i}^2  & \text{if } i=j\\
0 & \text{otherwise} \\
\end{cases}
$$  
Similarly, 
\begin{align*}
\frac{\partial F(\bp,\balpha)}{\partial \bp}\Big\vert_{(\bp,\balpha)=(p,\alpha)}= - \frac{\partial (\bp/\balpha)}{\partial \bp} \Big\vert_{(\bp,\balpha)=(p,\alpha)}=-\Lambda_1,
\end{align*}
since 
$$
\frac{\partial( \bp/\balpha)_{i}}{\partial\bp_{j}}= \begin{cases}\frac{\partial( \bp_{j}/\balpha_{j})_{}}{\partial \bp_{j}} & \text{if } i=j\\
0 & \text{otherwise} \\
\end{cases} = \begin{cases}1/\balpha_{i}  & \text{if } i=j\\
0 & \text{otherwise} \\
\end{cases}
$$
Thus, $F$ is continuously differentiable and $\frac{\partial F(\bp,\balpha)}{\partial \balpha}\vert_{(\bp,\balpha)=(p,\alpha)}=G\T G+\Lambda_0$ is invertible since $G\T G$ is positive semi-definite and $\Lambda_0$ is positive definite due to the condition of $p_i>0$ for all $i \in [K]$. Therefore, the function $F$  at  $(p,\alpha)$ satisfies the condition of the implicit function theorem, which implies  the  statement of this proposition as 
\begin{align}
\begin{split}
\frac{\partial \halpha(p)}{\partial p}=-\left[\frac{\partial F(\bar{p},\balpha)}{\partial \bar{\alpha}}\right]^{-1} \frac{\partial F(\bar{p},\balpha)}{\partial \bar{p}}\Big\vert_{(\bar{p},\balpha)=(p,\alpha)}
=\left[G\T G+\Lambda_0\right]^{-1} \Lambda_1. 
\end{split}
\end{align}

\end{proof}

\subsection{Proof of Theorem \ref{th:novergence}}

We adopt the following implicit assumption (or the design of algorithm) from \citep{navon2022multi}: if we reach a Pareto stationary solution at some step, the algorithm halts. We will also use the following Lemma.

\begin{lemma}\label{Kenji_lemma}
(Lemma A.1 of \citealp{navon2022multi})\ If $\mathcal{L}$ is differential and L-smooth (assumption \ref{assump:L-smooth}) then 
$\mathcal{L}(\theta')\leq \mathcal{L}(\theta)+\nabla\mathcal{L}(\theta)\T (\theta'-\theta)+\frac{L}{2}\|\theta'-\theta\|^{2}$.
\end{lemma}

\begin{proof}
We first note that if for some step, we reach a Pareto stationary solution the algorithm halts and sequence stays fixed at that point and therefore converges; Next, we assume that we never get to an exact Pareto stationary solution at any finite step.
Since $\sum_{i=1}^K p_{i}^{(t)}=1$ and  $\Dt^{(t)}=G^{(t)}\alpha^{(t)}$,
$$
||\Dt^{(t)}||^2=(\alpha^{(t)})\T ( G^{(t)})\T G^{(t)}\alpha^{(t)}=(\alpha^{(t)})\T(p^{(t)}/\alpha^{(t)})=\sum_{i=1}^K\alpha_i^{(t)}\cdot (p_{i}^{(t)}/\alpha_i^{(t)})=1.
$$ 
For each loss $\ell_i$  for $i\in[K]$,  using Lemma \ref{Kenji_lemma}
 and $(g_i^{(t)})\T\Dt^{(t)}=(g_i^{(t)})\T G^{(t)}\alpha^{(t)}=p_i^{(t)}/\alpha_i^{(t)}$,

\begin{align} 
    \ell_i(\theta^{(t+1)})&\leq \ell_i(\theta^{(t)})-  \mu^{(t)}\nabla\ell_i(\theta^{(t)})\T\Dt^{(t)}+\frac{L}{2}||\mu^{(t)}\Dt^{(t)}||^2
    \\ &= \ell_i(\theta^{(t)})- \mu^{(t)}\frac{p_i^{(t)}}{\alpha^{(t)}_i}+\frac{(\mu^{(t)})^2 L}{2}.
\end{align}\\ 
We average over the above inequality over all losses and get for $\mathcal{L}(\theta)=\frac{1}{K}\sum_{i=1}^K\ell_i(\theta)$:
\begin{align}
    &\mathcal{L}(\theta^{(t+1)})\leq \mathcal{L}(\theta^{(t)})- \mu^{(t)}\frac{1}{K}\sum_{i=1}^K\frac{p_i^{(t)}}{\alpha^{(t)}_i}+\frac{(\mu^{(t)})^2L}{2}=
    \mathcal{L}(\theta^{(t)})-L(\mu^{(t)})^2+\frac{(\mu^{(t)})^2L}{2}
    =\mathcal{L}(\theta^{(t)})-\frac{L(\mu^{(t)})^2}{2}. 
\end{align}
By rearranging, this shows that  $\mathcal{L}(\theta^{(t+1)}) \le\mathcal{L}(\theta^{(t)}) $ and   $\frac{L(\mu^{(t)})^2}{2}\le \mathcal{L}(\theta^{(t)})-\mathcal{L}(\theta^{(t+1)})$. From the first inequality,  the sequence   $\mathcal{(L}(\theta^{(t)}))_t$ is  non-increasing. As $\mathcal{L}(\theta^{(t)}) \in \RR$ is bounded below and $\mathcal{(L}(\theta^{(t)}))_t$ is non-increasing, the monotone convergence theorem concludes that the sequence $\mathcal{(L}(\theta^{(t)}))_t$ converges to a finite limit. Since a convergent sequence is a Cauchy sequence, $\frac{L(\mu^{(t)})^2}{2}\le \Lcal(\theta_{}^{t})-\Lcal(\theta_{}^{t+1}) \rightarrow 0$  as $t \rightarrow \infty$.
This implies that  $\mu^{(t)}\to 0$ as $t \rightarrow \infty$.
It follows that $\frac{1}{K}\sum_{i=1} ^{K}p_i^{(t)}/\alpha^{(t)}_i\to0$ as $t \rightarrow \infty$.
 Since $p_i^{(t)}>0$ and $\alpha^{(t)}_i>0$, it  implies that $\max_i p_i^{(t)}\geq 1/K$. If we define $j_t=\arg\max_i p_i^{(t)}$ we will get that ${p^{(t)}_{j_t}}/{\alpha^{(t)}_{j_t}}\rightarrow 0$ and therefore $\alpha^{(t)}_{j_t}\rightarrow \infty$. We can conclude that $||\alpha^{(t)}||\rightarrow\infty$.

We will now show that $||p^{(t)}/\alpha^{(t)}||$ is bounded for $t\to\infty$. As the sequence $\mathcal{L}(\theta^{(t)})$ is decreasing we have that the sequence $\theta^{(t)}$ is in the sublevel set $\{\theta:\mathcal{L}(\theta)\leq \mathcal{L}(\theta_0)\}$ which is closed and bounded and therefore compact. If follows that there exists $M<\infty$ such that $||g^{(t)}_i||\leq M$ for all $t$ and $i\in[K]$. We have for all $i$ and $t$, 
$|p_i^{(t)}/\alpha^{(t)}_i|=|(g^{(t)}_i)^T\Dt^{(t)}|\leq||g^{(t)}_i||\leq M <\infty$, and so $||p^{(t)}/\alpha^{(t)}||$ is bounded. 
Combining these two results we have $||p^{(t)}/\alpha^{(t)}||\geq \sigma_K((G^{(t)})\T G^{(t)})||\alpha^{(t)}||$ where $\sigma_K((G^{(t)})\T G^{(t)})$ is the smallest singular value of $(G^{(t)})\T G^{(t)}$.
Since the norm of $\alpha^{(t)}$ goes to infinity and the norm $p^{(t)}/\alpha^{(t)}$ is bounded, it follows that $\sigma_K((G^{(t)})\T G^{(t)})\to 0$.

Now, since $\{\theta : \Lcal(\theta)\leq \Lcal({\theta_0})\}$ is compact there exists a subsequence $\theta^{(t_j)}$ that converges to some point $\theta^*$. As $\sigma_K((G^{(t)})^T G^{(t)})\rightarrow 0$ we have from continuity that $\sigma_K(G_*\T G_*)=0$ where $G_*$ is the matrix of gradients at $\theta^*$. This means that the gradients at $\theta^*$ are linearly dependent and therefore $\theta^*$ is Pareto stationary by assumption \ref{assump:independence}. Since all loss functions  $\ell_i$ are the are differentiable, all loss functions $\ell_i$ are continuous. Thus, we have from the continuity of $\ell_i$ that  $\ell_i(\theta^{(t)})\xrightarrow{t\rightarrow\infty}\ell_i(\theta^*)$.
\end{proof}

\section{Solving $\mathbf{G\T G\alpha }=p/\alpha$}\label{app:solving-the-game}
Here, we describe how to efficiently approximate the optimal solution for $G\T G\alpha=p/\alpha$ through a sequence of convex optimization problems. We define a $\beta_i(\alpha)=g_i\T G\alpha$, and wish to find $\alpha$ such that $\alpha_i=p_i/\beta_i$ for all $i$, or equivalently $\log(\alpha_i)+\log(\beta_i(\alpha))-\log(p_i)=0$. Denote $\varphi_i(\alpha)=\log(\alpha_i)+\log(\beta_i)-\log(p_i)$ and $\varphi(\alpha)=\sum_i \varphi_i(\alpha)$. 
With that, our goal is to find a non-negative $\alpha$ such that $\varphi_i(\alpha)=0$ for all $i$. We can write this as the following optimization problem
\begin{align}
    \min_{\alpha} &\sum_i \varphi_i(\alpha) \\
    \text{s.t.} \forall i, \quad & -\varphi_i(\alpha) \leq 0 \nonumber\\
    &\alpha_i > 0 \nonumber \quad.
\end{align} 
One can see that the constraints of this problem are convex and linear while the objective is concave. To overcome this challenge we follow ~\citet{navon2022multi} and use the CCP to modify the concave objective into sequence of convex optimization problems. For further details please refer to Section 3.2 in ~\citet{navon2022multi}.

\section{Experimental Details} \label{app:exp_details}
Unless stated otherwise for \ourmethod{} we update the preference vector $p$ every $25$ optimization steps using SGD optimizer with momentum of $0.9$ and learning rate of $5e-3$.

\subsection{Illustrative Example}

We follow a similar setup as in~\citet{NavonAMCF21}. We consider a regression problem with parameters $W^T=(w_1, w_2)\in\RR^2$, shared among tasks, with no task-specific parameters. The optimal parameters for the main and helpful auxiliary tasks are ${W^\star}^T=(1, 1)$, while the optimal parameters for the harmful auxiliary are $\tilde{W}^T=(-1, -4)$. The main task is sampled from a Normal distribution $N({W^\star}^T x, \sigma_{\text{main}})$, with $\sigma_{\text{main}} = 20\cdot \sigma_{\text{h}}$ where $\sigma_{\text{h}}=0.25$ denotes the standard deviation for the noise of the helpful auxiliary. We use $1000$ training and train the model using Adam optimizer and learning-rate $1e-2$ and batch-size of $256$ for $1000$ epochs.

\subsection{CIFAR10-SSL}
We follow the training and evaluation procedure proposed by previous works \citep{shi2020auxiliary}. We use the CIFAR10 dataset and divide the dataset to train/val/test splits each containing $45K$/$5K$/$10K$ respectively. We allocate $5K$/$10K$ samples as our labeled dataset for the supervised main task. Following \citet{shi2020auxiliary} we use the wide resnet (WRN-28-2) architecture which is ResNet with depth 28 and width 2. We train WRN for $50K$ iterations using Adam optimizer with learning rate of $5e-3$ and $256$ batch size. Lastly, we use the main task performance on the validation set for early stopping.

\subsection{Speech Commands}
We used the Speech Command dataset, an audio dataset of spoken words designed to train and evaluate keyword spotting systems. We repeat the experiment twice with $1000$ and $500$ training samples distributed uniformly across $10$ classes. For validation and test we used the original dataset split containing $3643$ and $4107$ samples respectively. For all methods we use a CNN network with 3 Convolution layers with a linear layer as the classifier. We train the model for 200 epochs with Adam optimizer, and a learning rate of $1e-3$. We present the per-task results for both experiments in Table~\ref{tab:sc_500} and Table~\ref{tab:sc_1000}.

\begin{table}[]
\setlength{\tabcolsep}{3pt}
\small
    \centering
    \caption{\textit{Speech Commands}. Test performance on each one of the digit keywords. The train data-set consist of 1000 samples, uniformly distributed between the 10 classes. Values are averaged over 3 random seeds.}
    \vskip 0.11in
\begin{tabular}{lccccccccccc}
\midrule 
\multicolumn{12}{c}{Speech Commands - 1000} \\
\midrule
                  & 0     &  1     &  2     &  3     &  4     & 5    & 6     & 7     & 8     & 9     & \multicolumn{1}{l}{Mean} \\ \midrule
STL      & 97.0           & 96.7          & 95.8          & 96.3          & 96.4          & 96.3         & 96.7          & 96.6          & 96.0           & 96.6          & 96.4                            \\ \midrule
LS       & 97.0           & 97.2          & 96.2          & 97.1          & 97.0           & 96.2         & 96.9          & 97.3          & 95.8          & 96.8          & 96.7                            \\
PCGrad   & 97.0           & 97.3          & 96.2          & 96.6          & 96.7          & 96.7         & 96.8          & 97.3          & 95.8          & 96.9          & 96.7                            \\
CAGrad  & 97.2          & 97.7          & 95.6          & 97.0           & 97.0           & 96.6         & 97.0           & 97.2          & 95.8          & 97.0           & 96.8                            \\
Nash-MTL & 97.1          & 96.8          & 96.4          & 96.6          & 96.9          & 96.6         & 96.8          & 96.9          & 96.0           & 96.8          & 96.6                            \\ \midrule
GCS      & 97.3          & 97.8          & 96.2          & 97.0           & \textbf{97.4} & 96.2         & 96.8          & 97.3          & 96.2          & 97.1          & 96.9                            \\
OL-AUX   & 97.2          & 97.7          & 96.0           & 97.2          & 96.9          & 96.5         & 97.0           & 97.5          & 96.2          & 97.1          & 96.9                            \\ 
AuxiLearn   & 97.3          & 97.7          & 96.4           & 97.1          & 97.2          & 96.6         & 97.0           & 97.3          & 96.3          & 97.1          & 97.0                            \\ 
\midrule
AuxiNash & \textbf{97.4} & \textbf{98.1} & \textbf{96.5} & \textbf{97.2} & 97.1          & \textbf{97.0} & \textbf{97.1} & \textbf{97.6} & \textbf{96.3} & \textbf{97.6} & \textbf{97.2}  \\               \bottomrule
\end{tabular}
\label{tab:sc_1000}
\end{table}

\begin{table}[]
\setlength{\tabcolsep}{3pt}
\small
    \centering
    \caption{\textit{Speech Commands}. Test performance on each one of the digit keywords. The train data-set consist of 500 samples, uniformly distributed between the 10 classes. Values are averaged over 3 random seeds.}
    \vskip 0.11in
\begin{tabular}{lccccccccccc}
\midrule
\multicolumn{12}{c}{Speech Commands - 500} \\
\midrule \\
         & 0     & 1     & 2     & 3     & 4     & 5    & 6     & 7     & 8     & 9     & \multicolumn{1}{l}{Mean} \\ \midrule
STL      & 96.4          & 97.0           & 95.0           & 95.7          & 96.3          & 95.5         & 95.9          & 95.8          & 95.3          & 95.4          & 95.8                            \\ \midrule
LS       & 95.9          & 96.9          & 94.8          & 95.8          & 96.0           & 95.3         & 95.9          & 96.4          & 95.3          & 95.6          & 95.7                            \\
PCGrad   & 96.0           & 97.1          & 95.0           & 95.7          & 95.9          & 95.8         & 95.8          & 96.0           & 94.7          & 95.3          & 95.7                            \\
CAGrad  & 96.4          & 97.1          & 95.0           & 95.9          & 96.3          & 95.9         & 96.0           & 96.1          & 95.0           & 95.4          & 95.9                            \\
Nash-MTL & 96.6          & 97.4          & 94.1          & 95.7          & 96.3          & 95.4         & 95.9          & 96.2          & 94.8          & 95.2          & 95.7                            \\ \midrule
GCS      & 96.4          & 97.6          & \textbf{95.4} & 96.3          & 96.5          & 95.8         & 96.1          & 96.5          & 96.0           & 97.0           & 96.3                            \\
OL-AUX   & 96.2          & 97.2          & 94.6          & 96.0           & 96.4          & 95.7         & 95.9          & 96.4          & \textbf{96.1} & \textbf{97.5} & 96.2                             \\ 
AuxiLearn   & 96.2          & 97.5          & 94.8          & 96.3           & 96.3          & 95.8         & 96.1          & 96.3          & 95.3 & 95.8 & 96.0                             \\
\midrule
AuxiNash & \textbf{96.6} & \textbf{97.8} & 95.3          & \textbf{96.7} & \textbf{97.2} & \textbf{96.0} & \textbf{96.2} & \textbf{96.7} & 95.6          & 95.6          & \textbf{96.4} \\
\bottomrule
\end{tabular}
\label{tab:sc_500}
\end{table}

\subsection{Scene Understanding} We follow the training and evaluation protocol presented by previous studies \citep{Liu2019EndToEndML, liu2022auto, navon2022multi}. For all methods we train SegNet \citep{badrinarayanan2017segnet} model for $200$ epochs with Adam optimizer. We use learning rate of $1e-4$ for the first $100$ epochs, then reduced it to $5e-5$ for the remaining epochs. We use a batch size of $2$ and $8$ for NYUv2 and CityScapes respectively. During training we apply data augmentations for all compared methods, more specifically, we use random scale and random horizontal flip as augmentations. 
Following \citet{Liu2019EndToEndML, liu2022auto, navon2022multi} we report the test performance averaged over the last $10$ epochs.

\subsection{Controlling Task Preference}
We use the Multi-MNIST dataset that consits of $120K$ training samples and $20K$ testing samples. We allocate $12K$ training examples to form validation set. We use a variant of LeNet as the trained model. Specifically, the model contains $3$ CNN layers channels followed by $2$ fully connected layers. We train all methods using Adam optimizer with learning rate $1e-4$ for $50$ epochs and $256$ batch size. For \ourmethod{} we repeat this experiment $11$ times with varying preferences. More specifically, we select the first task preference from $p_1 \in \{ 0.01, 0.1, 0.2, 0.25, 0.4, 0.5, 0.6, 0.75, 0.8, 0.9, 0.99 \}$ and set $p_2=1-p_1$. We run the experiment equal amount of time with randomly selected seeds for Nash-MTL.

\section{Additional Results}

\subsection{Fixed Preference Vector}

\begin{table*}[!t]
\setlength{\tabcolsep}{3pt}
\small
    \centering
    \caption{\textit{NYUv2}. Test performance using fixed $p$. The preference for the main task, $p_{\text{main}}$, is in parentheses. Values are averaged over 3 random seeds.}
    \vskip 0.11in
\resizebox{0.95\textwidth}{!}{%
\begin{tabular}{cccccccccccccccccccc}
\toprule\\
 &  &  & \multicolumn{2}{c}{Segmentation} &  & \multicolumn{2}{c}{Depth} &  & \multicolumn{8}{c}{Surface Normal} &  &  &  \\
 \cmidrule(lr){4-5} \cmidrule(lr){7-8} \cmidrule(lr){10-17}
 &  &  & \multirow{2}{*}{mIoU $\uparrow$} & \multirow{2}{*}{Pix Acc $\uparrow$} &  & \multirow{2}{*}{Abs Err $\downarrow$} & \multirow{2}{*}{Rel Err $\downarrow$} &  & \multicolumn{2}{c}{Angle Distance $\downarrow$} &  & \multicolumn{5}{c}{Within $t^\circ$  $\uparrow$}  & $\mathbf{\Delta\%} \downarrow$ \\
 \cmidrule(lr){10-11} \cmidrule(lr){13-17} \cmidrule(lr){19-19} \cmidrule(lr){20-20}
 &  &  &  &  &  &  &  &  & Mean & Median &  & 11.25 &  & 22.5 &  & 30 \\
 \midrule
 & \multicolumn{2}{c}{STL} & $38.30$ & $63.76$ &  & $0.6754$ & $0.2780$ &  & $25.01$ & $19.21$ &  & $30.14$ &  & $57.20$ &  & $69.15$ \\
  \midrule
 & \multicolumn{2}{c}{AuxiNash $(0.7)$} & $\mathbf{41.36}$ & $66.48$ &  & $0.5214$ & $0.2138$ &  & $24.46$ & $18.78$ &  & $30.80$ &  & $58.33$ &  & $70.42$ & $\mathbf{-7.62}$  \\
 
 & \multicolumn{2}{c}{AuxiNash $(0.8)$} & $41.27$ & $66.47$ &  & $0.5322$ & $0.2172$ &  & $24.34$ & $18.67$ &  & $31.05$ &  & $58.53$ &  & $70.64$ & $-7.56$  \\
 
 & \multicolumn{2}{c}{AuxiNash $(0.9)$} & $41.35$ & $66.32$ &  & $0.5493$ & $0.2197$ &  & $\mathbf{24.39}$ & $\mathbf{18.61}$ &  & $\mathbf{31.23}$ &  & $\mathbf{58.64}$ &  & $\mathbf{70.68}$ & $-7.28$ \\
 \midrule
 & \multicolumn{2}{c}{\ourmethod{}} &
 $40.79$ & $\mathbf{66.79}$ &  & $\mathbf{0.5092}$ & $\mathbf{0.2042}$ & & $24.90$ & $19.31$ &  & $29.83$ &  & $57.07$ &  & $69.27$ & $-6.80$\\
 \bottomrule
\end{tabular}%
}
\label{tab:nyu-fixed-p}
\end{table*}

Our method, \ourmethod{}, dynamically adjusts the preference vector throughout the optimization process. It is possible, however, to fix the preference vector to its initial value and train a model over a grid of such preferences. As discussed in the main text, this procedure does not scale well as the  number of grid search values grows exponentially with the number of tasks. We also note that it is possible that the optimal preference changes during the optimization process, depending on the optimization dynamics. 

Here we present an ablation study of optimizing \ourmethod{} with fixed preference $p$. We use the NYUv2 dataset and train the model with $p_{\text{main}}\in\{0.9, 0.8, 0.7\}$. The preference for the two auxiliary tasks is equal, e.g., $p_{\text{aux},i}=0.1$ when $p_{\text{main}}=0.8$. The results are presented in Table~\ref{tab:nyu-fixed-p}.

\subsection{Sensitivity to Hyperparameters in Bi-level Optimization}
Bi-level optimization may be sensitive to hyperparameters (HPs), however, we show that this is not the case for AuxiNash. Generally, we did not tune HPs related to the inverse Hessian vector product approximation but used the ones reported by \cite{lorraine2020optimizing}. We find that these HPs generalize well to the various tasks that we studied without further tuning, suggesting that (good) HPs could transfer across tasks in our approach. For the remaining HP, namely the learning rate for updating the preference vector $p$, we fixed it to a single value of $5e-3$ throughout the experiments without tunning. Here we investigate the effect of the outer learning rate on the performance of AuxiNash. Specifically, we re-run the scene understanding experiment using NUYv2 dataset with $3$ outer learning rate values $\left\{ 1e-2, 5e-3, 1e-3 \right\}$. The results are presented in Table~\ref{tab:lr_ab} and averaged over $3$ seeds. The results in Table~\ref{tab:lr_ab} suggest that AuxiNash is not sensitive to the choice of learning rates, and significantly outperforms all baseline methods for a range of outer loop learning rate choices.

\begin{table*}[!h]
\setlength{\tabcolsep}{3pt}
\small
    \centering
    \caption{\textit{Outer learning rate ablation}. Test performance using varying outer learning rates. Values are averaged over 3 random seeds.}
    \vskip 0.11in
\resizebox{0.95\textwidth}{!}{
\begin{tabular}{cccccccccccccccccccc}
\toprule\\
 &  &  & \multicolumn{2}{c}{Segmentation} &  & \multicolumn{2}{c}{Depth} &  & \multicolumn{8}{c}{Surface Normal} &  &  &  \\
 \cmidrule(lr){4-5} \cmidrule(lr){7-8} \cmidrule(lr){10-17}
 &  &  & \multirow{2}{*}{mIoU $\uparrow$} & \multirow{2}{*}{Pix Acc $\uparrow$} &  & \multirow{2}{*}{Abs Err $\downarrow$} & \multirow{2}{*}{Rel Err $\downarrow$} &  & \multicolumn{2}{c}{Angle Distance $\downarrow$} &  & \multicolumn{5}{c}{Within $t^\circ$  $\uparrow$}  & $\mathbf{\Delta\%} \downarrow$ \\
 \cmidrule(lr){10-11} \cmidrule(lr){13-17} \cmidrule(lr){19-19} \cmidrule(lr){20-20}
 &  &  &  &  &  &  &  &  & Mean & Median &  & 11.25 &  & 22.5 &  & 30 \\
 \midrule
 & \multicolumn{2}{c}{AuxiNash $(lr=5e-3)$} & $40.79$ & $66.79$ &  & $0.5092$ & $0.2042$ &  & $25.01$ & $19.21$ &  & $30.14$ &  & $57.20$ &  & $69.15$ & $-6.80$ \\
  \midrule
 & \multicolumn{2}{c}{AuxiNash $(lr=1e-2)$} & $40.75$ & $66.74$ &  & $0.5175$ & $0.2055$ &  & $24.90$ & $19.31$ &  & $29.83$ &  & $57.07$ &  & $69.27$ & $-6.76$  \\
 & \multicolumn{2}{c}{AuxiNash $(lr=1e-3)$} & $41.03$ & $66.55$ &  & $0.5100$ & $0.2043$ &  & $25.08$ & $19.58$ &  & $29.34$ &  & $56.55$ &  & $68.83$ & $-6.22$  \\
 
 \bottomrule
\end{tabular}%
}
\label{tab:lr_ab}
\end{table*}

\end{document}

%% file: math_commands.tex
\usepackage{amsmath,amsfonts,bm}

\def\eqref#1{equation~\ref{#1}}

\def\1{\bm{1}}

\def\rb{{\textnormal{b}}}

\def\vb{{\bm{b}}}

\DeclareMathAlphabet{\mathsfit}{\encodingdefault}{\sfdefault}{m}{sl}
\SetMathAlphabet{\mathsfit}{bold}{\encodingdefault}{\sfdefault}{bx}{n}

\def\gL{{\mathcal{L}}}

\newcommand{\Ls}{\mathcal{L}}
\newcommand{\R}{\mathbb{R}}

\newcommand{\Dt}{\Delta\theta}

\let\ab\allowbreak

%% file: main_arxiv.bbl
\begin{thebibliography}{62}
\providecommand{\natexlab}[1]{#1}
\providecommand{\url}[1]{\texttt{#1}}
\expandafter\ifx\csname urlstyle\endcsname\relax
  \providecommand{\doi}[1]{doi: #1}\else
  \providecommand{\doi}{doi: \begingroup \urlstyle{rm}\Url}\fi

\bibitem[Achituve et~al.(2021)Achituve, Maron, and Chechik]{achituve2021self}
Achituve, I., Maron, H., and Chechik, G.
\newblock Self-supervised learning for domain adaptation on point clouds.
\newblock In \emph{Proceedings of the IEEE/CVF winter conference on
  applications of computer vision}, pp.\  123--133, 2021.

\bibitem[Badrinarayanan et~al.(2017)Badrinarayanan, Kendall, and
  Cipolla]{badrinarayanan2017segnet}
Badrinarayanan, V., Kendall, A., and Cipolla, R.
\newblock Segnet: A deep convolutional encoder-decoder architecture for image
  segmentation.
\newblock \emph{IEEE transactions on pattern analysis and machine
  intelligence}, 39\penalty0 (12):\penalty0 2481--2495, 2017.

\bibitem[Caruana(1997)]{caruana1997multitask}
Caruana, R.
\newblock Multitask learning.
\newblock \emph{Machine learning}, 28\penalty0 (1):\penalty0 41--75, 1997.

\bibitem[Chen et~al.()Chen, Wang, Liu, Zhou, Guan, and Zhu]{chenmodule}
Chen, H., Wang, X., Liu, Y., Zhou, Y., Guan, C., and Zhu, W.
\newblock Module-aware optimization for auxiliary learning.
\newblock In \emph{Advances in Neural Information Processing Systems}.

\bibitem[Chen et~al.(2022)Chen, Wang, Guan, Liu, and Zhu]{chen2022auxiliary}
Chen, H., Wang, X., Guan, C., Liu, Y., and Zhu, W.
\newblock Auxiliary learning with joint task and data scheduling.
\newblock In \emph{International Conference on Machine Learning}, pp.\
  3634--3647. PMLR, 2022.

\bibitem[Chen et~al.(2018)Chen, Badrinarayanan, Lee, and
  Rabinovich]{chen2017gradnorm}
Chen, Z., Badrinarayanan, V., Lee, C.-Y., and Rabinovich, A.
\newblock Gradnorm: Gradient normalization for adaptive loss balancing in deep
  multitask networks.
\newblock In \emph{International Conference on Machine Learning}, pp.\
  794--803. PMLR, 2018.

\bibitem[Chen et~al.(2020)Chen, Ngiam, Huang, Luong, Kretzschmar, Chai, and
  Anguelov]{Chen2020JustPA}
Chen, Z., Ngiam, J., Huang, Y., Luong, T., Kretzschmar, H., Chai, Y., and
  Anguelov, D.
\newblock Just pick a sign: Optimizing deep multitask models with gradient sign
  dropout.
\newblock \emph{ArXiv}, abs/2010.06808, 2020.

\bibitem[Cordts et~al.(2016)Cordts, Omran, Ramos, Rehfeld, Enzweiler, Benenson,
  Franke, Roth, and Schiele]{Cordts2016Cityscapes}
Cordts, M., Omran, M., Ramos, S., Rehfeld, T., Enzweiler, M., Benenson, R.,
  Franke, U., Roth, S., and Schiele, B.
\newblock The cityscapes dataset for semantic urban scene understanding.
\newblock In \emph{Proc. of the IEEE Conference on Computer Vision and Pattern
  Recognition (CVPR)}, 2016.

\bibitem[Dai et~al.(2016)Dai, He, and Sun]{dai2016instance}
Dai, J., He, K., and Sun, J.
\newblock Instance-aware semantic segmentation via multi-task network cascades.
\newblock In \emph{Proceedings of the IEEE conference on computer vision and
  pattern recognition}, pp.\  3150--3158, 2016.

\bibitem[de~Geus et~al.(2021)de~Geus, Meletis, Lu, Wen, and
  Dubbelman]{de2021part}
de~Geus, D., Meletis, P., Lu, C., Wen, X., and Dubbelman, G.
\newblock Part-aware panoptic segmentation.
\newblock In \emph{Proceedings of the IEEE/CVF Conference on Computer Vision
  and Pattern Recognition}, pp.\  5485--5494, 2021.

\bibitem[Dery et~al.(2022)Dery, Michel, Khodak, Neubig, and
  Talwalkar]{dery2022aang}
Dery, L.~M., Michel, P., Khodak, M., Neubig, G., and Talwalkar, A.
\newblock Aang: Automating auxiliary learning.
\newblock \emph{arXiv preprint arXiv:2205.14082}, 2022.

\bibitem[Du et~al.(2018)Du, Czarnecki, Jayakumar, Farajtabar, Pascanu, and
  Lakshminarayanan]{du2018adapting}
Du, Y., Czarnecki, W.~M., Jayakumar, S.~M., Farajtabar, M., Pascanu, R., and
  Lakshminarayanan, B.
\newblock Adapting auxiliary losses using gradient similarity.
\newblock \emph{arXiv preprint arXiv:1812.02224}, 2018.

\bibitem[Foo et~al.(2007)Foo, Ng, et~al.]{foo2007efficient}
Foo, C.-s., Ng, A., et~al.
\newblock Efficient multiple hyperparameter learning for log-linear models.
\newblock \emph{Advances in neural information processing systems}, 20, 2007.

\bibitem[Franceschi et~al.(2018)Franceschi, Frasconi, Salzo, Grazzi, and
  Pontil]{franceschi2018bilevel}
Franceschi, L., Frasconi, P., Salzo, S., Grazzi, R., and Pontil, M.
\newblock Bilevel programming for hyperparameter optimization and
  meta-learning.
\newblock In \emph{International Conference on Machine Learning}, pp.\
  1568--1577. PMLR, 2018.

\bibitem[Hashimoto et~al.(2017)Hashimoto, Xiong, Tsuruoka, and
  Socher]{hashimoto2017joint}
Hashimoto, K., Xiong, C., Tsuruoka, Y., and Socher, R.
\newblock A joint many-task model: Growing a neural network for multiple nlp
  tasks.
\newblock In \emph{Proceedings of the 2017 Conference on Empirical Methods in
  Natural Language Processing}, pp.\  1923--1933, 2017.

\bibitem[Hwang et~al.(2020)Hwang, Park, Kwon, Kim, Ha, and Kim]{HwangPKKHK20}
Hwang, D., Park, J., Kwon, S., Kim, K., Ha, J., and Kim, H.~J.
\newblock Self-supervised auxiliary learning with meta-paths for heterogeneous
  graphs.
\newblock In \emph{Advances in Neural Information Processing Systems
  {(NeurIPS)}}, 2020.

\bibitem[Javaloy \& Valera(2021)Javaloy and Valera]{javaloy2021rotograd}
Javaloy, A. and Valera, I.
\newblock Rotograd: Gradient homogenization in multitask learning.
\newblock \emph{arXiv preprint arXiv:2103.02631}, 2021.

\bibitem[Kalai(1977)]{kalai1977nonsymmetric}
Kalai, E.
\newblock Nonsymmetric nash solutions and replications of 2-person bargaining.
\newblock \emph{International Journal of Game Theory}, 6\penalty0 (3):\penalty0
  129--133, 1977.

\bibitem[Kingma \& Ba(2015)Kingma and Ba]{Kingma2015AdamAM}
Kingma, D.~P. and Ba, J.
\newblock Adam: A method for stochastic optimization.
\newblock \emph{CoRR}, abs/1412.6980, 2015.

\bibitem[Kung et~al.(2021)Kung, Yin, Chen, Yang, and Chen]{kung2021efficient}
Kung, P.-N., Yin, S.-S., Chen, Y.-C., Yang, T.-H., and Chen, Y.-N.
\newblock Efficient multi-task auxiliary learning: Selecting auxiliary data by
  feature similarity.
\newblock In \emph{Proceedings of the 2021 Conference on Empirical Methods in
  Natural Language Processing}, pp.\  416--428, 2021.

\bibitem[Liao et~al.(2018)Liao, Xiong, Fetaya, Zhang, Yoon, Pitkow, Urtasun,
  and Zemel]{liao2018reviving}
Liao, R., Xiong, Y., Fetaya, E., Zhang, L., Yoon, K., Pitkow, X., Urtasun, R.,
  and Zemel, R.
\newblock Reviving and improving recurrent back-propagation.
\newblock In \emph{International Conference on Machine Learning}, pp.\
  3082--3091. PMLR, 2018.

\bibitem[Lin et~al.(2019)Lin, Baweja, Kantor, and Held]{lin2019adaptive}
Lin, X., Baweja, H., Kantor, G., and Held, D.
\newblock Adaptive auxiliary task weighting for reinforcement learning.
\newblock \emph{Advances in neural information processing systems}, 32, 2019.

\bibitem[Lipp \& Boyd(2016)Lipp and Boyd]{lipp2016variations}
Lipp, T. and Boyd, S.
\newblock Variations and extension of the convex--concave procedure.
\newblock \emph{Optimization and Engineering}, 17\penalty0 (2):\penalty0
  263--287, 2016.

\bibitem[Liu et~al.(2021{\natexlab{a}})Liu, Liu, Jin, Stone, and
  Liu]{liu2021conflict}
Liu, B., Liu, X., Jin, X., Stone, P., and Liu, Q.
\newblock Conflict-averse gradient descent for multi-task learning.
\newblock \emph{Advances in Neural Information Processing Systems}, 34,
  2021{\natexlab{a}}.

\bibitem[Liu et~al.(2021{\natexlab{b}})Liu, Li, Kuang, Xue, Chen, Yang, Liao,
  and Zhang]{liu2021towards}
Liu, L., Li, Y., Kuang, Z., Xue, J., Chen, Y., Yang, W., Liao, Q., and Zhang,
  W.
\newblock Towards impartial multi-task learning.
\newblock ICLR, 2021{\natexlab{b}}.

\bibitem[Liu et~al.(2021{\natexlab{c}})Liu, Gao, Zhang, Meng, and
  Lin]{liu2021investigating}
Liu, R., Gao, J., Zhang, J., Meng, D., and Lin, Z.
\newblock Investigating bi-level optimization for learning and vision from a
  unified perspective: A survey and beyond.
\newblock \emph{IEEE Transactions on Pattern Analysis and Machine
  Intelligence}, 2021{\natexlab{c}}.

\bibitem[Liu et~al.(2019{\natexlab{a}})Liu, Davison, and Johns]{liu2019self}
Liu, S., Davison, A., and Johns, E.
\newblock Self-supervised generalisation with meta auxiliary learning.
\newblock \emph{Advances in Neural Information Processing Systems}, 32,
  2019{\natexlab{a}}.

\bibitem[Liu et~al.(2019{\natexlab{b}})Liu, Johns, and
  Davison]{Liu2019EndToEndML}
Liu, S., Johns, E., and Davison, A.~J.
\newblock End-to-end multi-task learning with attention.
\newblock \emph{2019 IEEE/CVF Conference on Computer Vision and Pattern
  Recognition (CVPR)}, pp.\  1871--1880, 2019{\natexlab{b}}.

\bibitem[Liu et~al.(2022)Liu, James, Davison, and Johns]{liu2022auto}
Liu, S., James, S., Davison, A.~J., and Johns, E.
\newblock Auto-lambda: Disentangling dynamic task relationships.
\newblock \emph{arXiv preprint arXiv:2202.03091}, 2022.

\bibitem[Lorraine et~al.(2020)Lorraine, Vicol, and
  Duvenaud]{lorraine2020optimizing}
Lorraine, J., Vicol, P., and Duvenaud, D.
\newblock Optimizing millions of hyperparameters by implicit differentiation.
\newblock In \emph{International Conference on Artificial Intelligence and
  Statistics}, pp.\  1540--1552. PMLR, 2020.

\bibitem[Luketina et~al.(2016)Luketina, Berglund, Greff, and
  Raiko]{luketina2016scalable}
Luketina, J., Berglund, M., Greff, K., and Raiko, T.
\newblock Scalable gradient-based tuning of continuous regularization
  hyperparameters.
\newblock In \emph{International conference on machine learning}, pp.\
  2952--2960. PMLR, 2016.

\bibitem[MacKay et~al.(2019)MacKay, Vicol, Lorraine, Duvenaud, and
  Grosse]{mackay2019self}
MacKay, M., Vicol, P., Lorraine, J., Duvenaud, D., and Grosse, R.
\newblock Self-tuning networks: Bilevel optimization of hyperparameters using
  structured best-response functions.
\newblock \emph{arXiv preprint arXiv:1903.03088}, 2019.

\bibitem[Maninis et~al.(2019)Maninis, Radosavovic, and
  Kokkinos]{maninis2019attentive}
Maninis, K.-K., Radosavovic, I., and Kokkinos, I.
\newblock Attentive single-tasking of multiple tasks.
\newblock In \emph{Proceedings of the IEEE/CVF Conference on Computer Vision
  and Pattern Recognition}, pp.\  1851--1860, 2019.

\bibitem[Misra et~al.(2016)Misra, Shrivastava, Gupta, and
  Hebert]{misra2016cross}
Misra, I., Shrivastava, A., Gupta, A., and Hebert, M.
\newblock Cross-stitch networks for multi-task learning.
\newblock In \emph{Proceedings of the IEEE conference on computer vision and
  pattern recognition}, pp.\  3994--4003, 2016.

\bibitem[Nash(1953)]{nash}
Nash, J.
\newblock Two-person cooperative games.
\newblock \emph{Econometrica}, 21\penalty0 (1):\penalty0 128--140, 1953.
\newblock ISSN 00129682, 14680262.
\newblock URL \url{http://www.jstor.org/stable/1906951}.

\bibitem[Navon et~al.(2021)Navon, Achituve, Maron, Chechik, and
  Fetaya]{NavonAMCF21}
Navon, A., Achituve, I., Maron, H., Chechik, G., and Fetaya, E.
\newblock Auxiliary learning by implicit differentiation.
\newblock In \emph{International Conference on Learning Representations
  {(ICLR)}}, 2021.

\bibitem[Navon et~al.(2022)Navon, Shamsian, Achituve, Maron, Kawaguchi,
  Chechik, and Fetaya]{navon2022multi}
Navon, A., Shamsian, A., Achituve, I., Maron, H., Kawaguchi, K., Chechik, G.,
  and Fetaya, E.
\newblock Multi-task learning as a bargaining game.
\newblock In \emph{International Conference on Machine Learning}, 2022.

\bibitem[Oliver et~al.(2018)Oliver, Odena, Raffel, Cubuk, and
  Goodfellow]{oliver2018realistic}
Oliver, A., Odena, A., Raffel, C.~A., Cubuk, E.~D., and Goodfellow, I.
\newblock Realistic evaluation of deep semi-supervised learning algorithms.
\newblock \emph{Advances in neural information processing systems}, 31, 2018.

\bibitem[Pedregosa(2016)]{pedregosa2016hyperparameter}
Pedregosa, F.
\newblock Hyperparameter optimization with approximate gradient.
\newblock In \emph{International conference on machine learning}, pp.\
  737--746. PMLR, 2016.

\bibitem[Pinto \& Gupta(2017)Pinto and Gupta]{pinto2017learning}
Pinto, L. and Gupta, A.
\newblock Learning to push by grasping: Using multiple tasks for effective
  learning.
\newblock In \emph{2017 IEEE international conference on robotics and
  automation (ICRA)}, pp.\  2161--2168. IEEE, 2017.

\bibitem[Raghu et~al.(2020)Raghu, Raghu, Kornblith, Duvenaud, and
  Hinton]{raghu2020teaching}
Raghu, A., Raghu, M., Kornblith, S., Duvenaud, D., and Hinton, G.
\newblock Teaching with commentaries.
\newblock \emph{arXiv preprint arXiv:2011.03037}, 2020.

\bibitem[Rajeswaran et~al.(2019)Rajeswaran, Finn, Kakade, and
  Levine]{Rajeswaran2019MetaLearningWI}
Rajeswaran, A., Finn, C., Kakade, S.~M., and Levine, S.
\newblock Meta-learning with implicit gradients.
\newblock In \emph{Neural Information Processing Systems}, 2019.

\bibitem[Ruder(2017)]{ruder2017overview}
Ruder, S.
\newblock An overview of multi-task learning in deep neural networks.
\newblock \emph{arXiv preprint arXiv:1706.05098}, 2017.

\bibitem[Sabour et~al.(2017)Sabour, Frosst, and Hinton]{sabour2017dynamic}
Sabour, S., Frosst, N., and Hinton, G.~E.
\newblock Dynamic routing between capsules.
\newblock \emph{Advances in neural information processing systems}, 30, 2017.

\bibitem[Schaul et~al.(2019)Schaul, Borsa, Modayil, and Pascanu]{schaul2019ray}
Schaul, T., Borsa, D., Modayil, J., and Pascanu, R.
\newblock Ray interference: a source of plateaus in deep reinforcement
  learning.
\newblock \emph{arXiv preprint arXiv:1904.11455}, 2019.

\bibitem[Sener \& Koltun(2018)Sener and Koltun]{sener2018multi}
Sener, O. and Koltun, V.
\newblock Multi-task learning as multi-objective optimization.
\newblock In \emph{Advances in Neural Information Processing Systems}, pp.\
  527--538, 2018.

\bibitem[Shi et~al.(2020)Shi, Hoffman, Saenko, Darrell, and
  Xu]{shi2020auxiliary}
Shi, B., Hoffman, J., Saenko, K., Darrell, T., and Xu, H.
\newblock Auxiliary task reweighting for minimum-data learning.
\newblock \emph{Advances in Neural Information Processing Systems},
  33:\penalty0 7148--7160, 2020.

\bibitem[Silberman et~al.(2012)Silberman, Hoiem, Kohli, and
  Fergus]{silberman2012indoor}
Silberman, N., Hoiem, D., Kohli, P., and Fergus, R.
\newblock Indoor segmentation and support inference from rgbd images.
\newblock In \emph{European conference on computer vision}, pp.\  746--760.
  Springer, 2012.

\bibitem[Sinha et~al.(2017)Sinha, Malo, and Deb]{sinha2017review}
Sinha, A., Malo, P., and Deb, K.
\newblock A review on bilevel optimization: from classical to evolutionary
  approaches and applications.
\newblock \emph{IEEE Transactions on Evolutionary Computation}, 22\penalty0
  (2):\penalty0 276--295, 2017.

\bibitem[Standley et~al.(2020)Standley, Zamir, Chen, Guibas, Malik, and
  Savarese]{StandleyZCGMS20}
Standley, T., Zamir, A., Chen, D., Guibas, L.~J., Malik, J., and Savarese, S.
\newblock Which tasks should be learned together in multi-task learning?
\newblock In \emph{International Conference on Machine Learning ({ICML})},
  2020.

\bibitem[Szép \& Forgó(1985)Szép and Forgó]{game_theory}
Szép, J. and Forgó, F.
\newblock \emph{Introduction to the Theory of Games}.
\newblock Springer, 1985.

\bibitem[Vicol et~al.(2022)Vicol, Lorraine, Pedregosa, Duvenaud, and
  Grosse]{vicol2022implicit}
Vicol, P., Lorraine, J.~P., Pedregosa, F., Duvenaud, D., and Grosse, R.~B.
\newblock On implicit bias in overparameterized bilevel optimization.
\newblock In \emph{International Conference on Machine Learning}, pp.\
  22234--22259. PMLR, 2022.

\bibitem[Wang et~al.(2020)Wang, Tsvetkov, Firat, and Cao]{wang2020gradient}
Wang, Z., Tsvetkov, Y., Firat, O., and Cao, Y.
\newblock Gradient vaccine: Investigating and improving multi-task optimization
  in massively multilingual models.
\newblock In \emph{International Conference on Learning Representations}, 2020.

\bibitem[Warden(2018)]{warden2018speech}
Warden, P.
\newblock Speech commands: A dataset for limited-vocabulary speech recognition.
\newblock \emph{arXiv preprint arXiv:1804.03209}, 2018.

\bibitem[Wen et~al.(2020)Wen, Zhang, Wang, Lv, Bao, Lin, and
  Yang]{wen2020entire}
Wen, H., Zhang, J., Wang, Y., Lv, F., Bao, W., Lin, Q., and Yang, K.
\newblock Entire space multi-task modeling via post-click behavior
  decomposition for conversion rate prediction.
\newblock In \emph{Proceedings of the 43rd International ACM SIGIR conference
  on research and development in Information Retrieval}, pp.\  2377--2386,
  2020.

\bibitem[Yang et~al.(2019)Yang, Chen, Hong, and Wang]{yang2019provably}
Yang, Z., Chen, Y., Hong, M., and Wang, Z.
\newblock Provably global convergence of actor-critic: A case for linear
  quadratic regulator with ergodic cost.
\newblock \emph{Advances in neural information processing systems}, 32, 2019.

\bibitem[Ying(2019)]{ying2019overview}
Ying, X.
\newblock An overview of overfitting and its solutions.
\newblock In \emph{Journal of physics: Conference series}, volume 1168, pp.\
  022022. IOP Publishing, 2019.

\bibitem[Yu et~al.(2020)Yu, Kumar, Gupta, Levine, Hausman, and
  Finn]{yu2020gradient}
Yu, T., Kumar, S., Gupta, A., Levine, S., Hausman, K., and Finn, C.
\newblock Gradient surgery for multi-task learning.
\newblock In \emph{Advances in Neural Information Processing Systems}, 2020.

\bibitem[Yuille \& Rangarajan(2003)Yuille and Rangarajan]{yuille2003concave}
Yuille, A.~L. and Rangarajan, A.
\newblock The concave-convex procedure.
\newblock \emph{Neural computation}, 15\penalty0 (4):\penalty0 915--936, 2003.

\bibitem[Zhai et~al.(2019)Zhai, Oliver, Kolesnikov, and Beyer]{zhai2019s4l}
Zhai, X., Oliver, A., Kolesnikov, A., and Beyer, L.
\newblock S4l: Self-supervised semi-supervised learning.
\newblock In \emph{Proceedings of the IEEE/CVF International Conference on
  Computer Vision}, pp.\  1476--1485, 2019.

\bibitem[Zhang et~al.(2020)Zhang, Chen, Huang, Li, Yang, Zhang, and
  Wang]{zhang2020bi}
Zhang, H., Chen, W., Huang, Z., Li, M., Yang, Y., Zhang, W., and Wang, J.
\newblock Bi-level actor-critic for multi-agent coordination.
\newblock In \emph{Proceedings of the AAAI Conference on Artificial
  Intelligence}, volume~34, pp.\  7325--7332, 2020.

\bibitem[Zhang et~al.(2014)Zhang, Luo, Loy, and Tang]{zhang2014facial}
Zhang, Z., Luo, P., Loy, C.~C., and Tang, X.
\newblock Facial landmark detection by deep multi-task learning.
\newblock In \emph{European conference on computer vision}, pp.\  94--108.
  Springer, 2014.

\end{thebibliography}
